\documentclass[sigconf]{acmart}

\AtBeginDocument{%
  \providecommand\BibTeX{{%
    \normalfont B\kern-0.5em{\scshape i\kern-0.25em b}\kern-0.8em\TeX}}}

\setcopyright{acmcopyright}
\copyrightyear{2018}
\acmYear{2018}
\acmDOI{XXXXXXX.XXXXXXX}

\acmConference[Conference acronym 'XX]{Make sure to enter the correct
  conference title from your rights confirmation emai}{June 03--05,
  2018}{Woodstock, NY}
\acmPrice{15.00}
\acmISBN{978-1-4503-XXXX-X/18/06}

\usepackage{hyperref}

\usepackage{url}
\usepackage{booktabs}       
\usepackage{amsfonts}       
\usepackage{nicefrac}       
\usepackage{microtype}      
\usepackage{xcolor}         
\usepackage{amsmath} 
\usepackage{amsthm,latexsym,paralist}
\usepackage{wrapfig}
\usepackage{graphicx}
\usepackage{subcaption}
\usepackage{multirow}
\usepackage{bm}
\usepackage{xspace}

\usepackage{paralist}

\theoremstyle{plain}
\newtheorem{theorem}{Theorem}
\newenvironment{theorem-again}[1]
  {\renewcommand{\thetheorem}{\ref{#1}*}%
   \addtocounter{theorem}{-1}%
   \begin{theorem}}
  {\end{theorem}}
\newtheorem{proposition}{Proposition}

\theoremstyle{definition}

\newtheorem{assumption}{Assumption}

\newenvironment{assumption-again}[1]
  {\renewcommand{\theassumption}{\ref{#1}*}%
   \addtocounter{assumption}{-1}%
   \begin{assumption}}
  {\end{assumption}}

\theoremstyle{remark}

\usepackage{amsmath}

\newcommand{\hz}[1]{\textcolor{red}{@HZ:~#1}}

\newcommand{\method}{\textsc{DGat}\xspace}

\usepackage{colortbl}
\definecolor{Gray}{gray}{0.9}

\begin{document}

\title{Investigating Out-of-Distribution Generalization of GNNs:\\ An Architecture Perspective}

%
\author{Kai Guo}
\affiliation{%
  \institution{School of Artificial Intelligence,}
  \institution{Jilin University}
  \country{}
}
\email{guokai20@mails.jlu.edu.cn}

\author{Hongzhi Wen}
\affiliation{%
    \institution{Department of Computer Science and Engineering,}
  \institution{Michigan State University}
  \country{}
}
\email{wenhongz@msu.edu}

\author{Wei Jin}
\affiliation{%
\institution{Department of Computer Science,}
  \institution{Emory University}
  \country{}
}
\email{wei.jin@emory.edu}

\author{Yaming Guo}
\affiliation{%
  \institution{School of Artificial Intelligence,}
  \institution{Jilin University}
  \country{}
}
\email{guoym21@mails.jlu.edu.cn}

\author{Jiliang Tang}
\affiliation{
\institution{Department of Computer Science and Engineering,}
  \institution{Michigan State University}
  \country{}
}
\email{tangjili@msu.edu}

\author{Yi Chang}
\affiliation{%
  \institution{School of Artificial Intelligence,}
  \institution{Jilin University}
  \country{}
}
\email{yichang@jlu.edu.cn}
\renewcommand{\shortauthors}{Trovato and Tobin, et al.}

\begin{abstract}
Graph neural networks (GNNs) have exhibited remarkable performance under the assumption that test data comes from the same distribution of training data. However, in real-world scenarios, this assumption may not always be valid. Consequently, there is a growing focus on exploring the Out-of-Distribution (OOD) problem in the context of graphs. Most existing efforts have primarily concentrated on improving graph OOD generalization from two \textbf{model-agnostic} perspectives: data-driven methods and strategy-based learning. However, there has been limited attention dedicated to investigating the impact of well-known \textbf{GNN model architectures} on graph OOD generalization, which is orthogonal to existing research. In this work, we provide the first comprehensive investigation of OOD generalization on graphs from an architecture perspective, by examining the common building blocks of modern GNNs. Through extensive experiments, we reveal that both the graph self-attention mechanism and the decoupled architecture contribute positively to graph OOD generalization. In contrast, we observe that the linear classification layer tends to compromise graph OOD generalization capability. Furthermore, we provide in-depth theoretical insights and discussions to underpin these discoveries. 
These insights have empowered us to develop a novel GNN backbone model,  \method, designed to harness the robust properties of both graph self-attention mechanism and the decoupled architecture.  Extensive experimental results demonstrate the effectiveness of our model under graph OOD, exhibiting substantial and consistent enhancements across various training strategies.
\end{abstract}

\ccsdesc[500]{Computing methodologies~Neural networks}


\keywords{out-of-distribution generalization, graph neural networks, self-attention, decoupled architecture}


\received{20 February 2007}
\received[revised]{12 March 2009}
\received[accepted]{5 June 2009}

\maketitle

\section{Introduction}
Graph Neural Networks (GNNs) have emerged as a powerful tool for representation learning on various graph-structured data, such as social networks~\cite{social1,social2,social3}, citation networks~\cite{kipf,citation1,zhou2021dirichlet}, biological networks~\cite{wen2022graph,wang2021leverage}, and product graphs~\cite{product1,product3}. The representations learned by GNNs can tremendously facilitate diverse downstream tasks, including node classification~\cite{node1,node2,mao2021source}, graph classification~\cite{graph1,sui2022causal}, and link prediction~\cite{link1,link2}. In many of these tasks, it is conventionally assumed that training and test datasets are drawn from an identical distribution. Nonetheless, this assumption is often contravened in practical scenarios~\cite{ood3,eerm}. For instance, for paper classification on citation graphs, models may be trained on papers from a specific timeframe but later be required to make predictions for more recent publications~\cite{eerm}. Such discrepancies between training and test data distributions outline the out-of-distribution (OOD) challenge.


Recent studies on GNNs have pointed out potential vulnerabilities when these models face distributional shifts~\cite{ood1, ood2, ood3, ood4}. To counteract this, existing techniques aim to enhance graph OOD generalization for node classification tasks majorly from two perspectives: data-based and learning-strategy-based methods. Data-based methods focus on manipulating the input graph data to boost OOD generalization. Examples of such strategies include graph data augmentation~\cite{gui2022good,zhao2022graph} and graph transformation~\cite{GTrans}. On the other hand, learning-strategy-based methods emphasize modifying training approaches by introducing specialized optimization objectives and constraints. Notable methods encompass graph invariant learning~\cite{Lisa,eerm} and regularization techniques~\cite{zhu2021shift}. When integrated with existing GNN backbone models~\cite{kipf,velivckovic2017graph,wu2019simplifying,gasteiger2018predict}, these methods can enhance their OOD generalization capabilities.

While these techniques have made strides in addressing graph OOD challenges, they primarily focus on external techniques for improving OOD generalization.
Moreover, there is a phenomenon where different GNN models exhibit varying performance in graph OOD scenarios\cite{GTrans}. \textbf{Hence, there remains a significant gap in our understanding of the inherent OOD generalization capabilities of different GNN backbone architectures. }
This lack of insight raises critical questions: \textit{Which GNN architecture is best suited when dealing with OOD scenarios? Are some models naturally more robust than others?} There is a pressing need to delve deeper into these architectures, comprehensively assess their innate capabilities, and provide clearer guidelines for their deployment in OOD situations.

\vskip 0.3em
\noindent \textbf{Research Questions.} To bridge the gap, this paper presents the first systematic examination on  OOD generalization abilities of popular GNN architectures, specifically targeting the node classification task. Our analysis is framed around three key questions:
\begin{compactenum}[Q1:]
\item How do distinct GNN architectures perform when exposed to OOD data?
\item If there are performance differences as indicated in Q1, can we identify specific designs within GNN architectures responsible for these variations? What could be the underlying causes?
\item Informed by the insights from the previous questions, can we develop new GNN designs that enhance OOD generalization?
\end{compactenum}

\vskip 0.3em
\noindent\textbf{Contributions.} By addressing the above questions, our contributions are summarized as follows:
\begin{compactenum}[\text{A}1:]
\item  We rigorously assess a set of common building modules of GNNs, including attention mechanism, decoupled architecture and linear classification layer.
Our empirical investigation reveals that both the attention mechanism and the decoupled architecture contribute positively to OOD generalization. Conversely, we observe that the linear classification layer tends to impair the OOD generalization capability. 
\item For a deeper understanding, we delve into the reasons why certain building modules enhance OOD generalization and provide corresponding analysis. We demonstrate that the graph self-attention mechanism in graphs adhering to the information bottleneck principle is beneficial for OOD generalization.
\item Based on our findings, we introduce a novel GNN design, Decoupled Graph Attention Network (\method), which combines the attention mechanism with the decoupled architecture, enabling dynamical adjustments of the propagation weights and separating propagation from feature transformation. 
\end{compactenum}
Our major contribution lies in \textbf{the systematic investigation of the impact of GNN architectural modules} on OOD scenarios. Remarkably, \textbf{our study is orthogonal to existing research efforts of model-agnostic solutions}. Indeed, our findings can complement existing external strategies. \method{} achieves superior performance against other GNN backbone models when trained using various OOD algorithms.

\section{Preliminaries}

\subsection{Graph OOD Generalization Problem}

The aim of our study is to investigate the out-of-distribution (OOD) generalization problem in graph domain from an underexplored perspective, the GNN backbone models. As a preliminary, we first introduce the graph OOD problem, and then discuss representative backbones for graph OOD generalization.

The OOD problem originates in the distribution shifts between training and test data. In a supervised learning setting, such distribution shifts can be defined as two types: i.e., covariate shift and concept shift~\cite{gui2022good}.  

In the realm of graph OOD, the input $X$ is specified as a graph $\mathcal{G}=(\mathcal{V}, \mathcal{E})$, with $N$ nodes $v_i \in \mathcal{V}$, edges $\left(v_i, v_j\right) \in \mathcal{E}$, an adjacency matrix $\mathbf{A} \in \mathbb{R}^{N \times N}$. Therefore, the covariate variable consists of an input pair of $(\mathbf{X}, \mathbf{A})$, where $\mathbf{X} \in \mathbb{R}^{N \times d}$ now denotes the node feature matrix. Consequently, graph OOD problems involve distribution shifts with both $\mathbf{X}$ and $\mathbf{A}$, which is more intricate than the general OOD problem. Graph distribution shifts can also be defined as two types: i.e., covariate shift and concept shift. Covariate shift refers to the distribution shift in input variables between training and test data. Formally, $P_{\text {tr}}(\mathbf{X},\mathbf{A}) \neq P_{\text {te}}(\mathbf{X},\mathbf{A})$. On the other hand, concept shift depicts the shift in conditional distribution $P(Y \mid \mathbf{X},\mathbf{A})$ between training and test data, i.e., $P_{\text {tr}}(Y \mid \mathbf{X},\mathbf{A}) \neq P_{\text {te}}(Y \mid \mathbf{X},\mathbf{A})$.

The aim of graph OOD generalization is to bolster model performance in OOD scenarios. To solve an OOD problem, we deploy a GNN backbone model, symbolized as a mapping function $f_\theta({\bf X},{\bf A})$ with learnable parameter $\theta$. The graph OOD generalization problem can be articulated as:
\begin{equation}
\min _\theta \max _{e \in \mathcal{E}} \mathbb{E}_{({\bf X},{\bf A},Y) \sim p({\bf X},{\bf A}, Y \mid e=e)}[\mathcal{L}(f_\theta({\bf X},{\bf A}), Y) \mid e]
\end{equation}
where $\mathcal{L}(\cdot)$ represents a loss function, and $e$ denotes the environment. In the graph machine learning research community, much attention has been dedicated~\cite{eerm, sagawa2019distributionally} to refining loss function $\mathcal{L}$, optimizing $\theta$, or augmenting on $\mathbf{X}$ and $\mathbf{A}$. However, our approach diverges from these methods. Instead, we focus on examining the impact of the design choices in the backbone model $f$. To the best of our knowledge, our study offers the first systematic analysis of GNN backbone architectures' effects in OOD contexts.

\subsection{Graph Neural Network Architectures} \label{sec:pregnn}

Next, we briefly compare the architectures of classic GNN models that are investigated in our analysis in Section~\ref{sec:ana}. 

\textbf{GCN}~\citep{kipf}. The Graph Convolutional Network (GCN) stands as the most representative model, and is chosen as the standard baseline model in our study. A graph convolutional layer comprises a pair of aggregation and transformation operators. At the $l$-th layer, the mathematical representation of the graph convolutional layer is as follows:
\begin{equation}
\mathbf{Z}^{(l)}=\sigma(\hat{\mathbf{A}} \mathbf{Z}^{(l-1)} \mathbf{W}^{(l)}),
\label{eq:GCN}
\end{equation}
where $\mathbf{Z}^{(l)}$ denotes the node representation of $l-th$ layer. $\mathbf{W}^{(l)}$ represents the linear transformation matrix. $\sigma(\cdot)$ is the nonlinear activation function. Note that $\hat{\mathbf{A}}$ is the normalized adjacency matrix, calculated from $\hat{\mathbf{A}} = \tilde{\mathbf{D}}^{-\frac{1}{2}}\tilde{\mathbf{A}}\tilde{\mathbf{D}}^{-\frac{1}{2}}$, where $\tilde{\mathbf{A}}=\mathbf{A}+\mathbf{I}$ is the adjacency matrix with added self-loops, $\tilde{\mathbf{D}}_{i i}=\sum_j \tilde{\mathbf{A}}_{i j}$ is the degree matrix. 

When applying the GCN model, there are two prevalent implementations. The first~\cite{gui2022good} adds a linear prediction layer after the final graph convolutional layer (noted as layer $L$), projecting the dimension of $\mathbf{Z}^{(L)}$ to that of the label $Y$. The second~\cite{kipf}, noted as \textbf{GCN--}, omits this prediction layer, adjusting dimensions directly in the last convolutional layer by altering the dimension of $\mathbf{W}^{(L)}$. While the distinction between these methods might be overlooked, it can influence OOD performance. We have compared both implementations in Section~\ref{sec:ana}, and our findings and analysis are detailed in Section~\ref{sec:lin}.

\textbf{GAT}~\citep{velivckovic2017graph}. In contrast to the fixed aggregation in GCN, Graph Attention Networks (GAT) employ a self-attention mechanism to assign distinct weights to neighboring nodes during aggregation. The attention mechanism can be expressed as:
\begin{equation}\label{eqn: GAT}
\alpha_{i j}=\frac{\exp \left(\sigma\left(\mathbf{a}^T\left[\mathbf{W} \mathbf{Z}_i \| \mathbf{W} \mathbf{Z}_j\right]\right)\right)}{\sum_{k \in \mathcal{N}_i} \exp \left(\sigma\left(\mathbf{a}^T\left[\mathbf{W} \mathbf{Z}_i \| \mathbf{W} \mathbf{Z}_k\right]\right)\right)},
\end{equation}
where $\mathbf{a}$ is a learnable weight vector, $\mathbf{Z}_i$ is the representation vector of node $i$, and $\mathcal{N}_i$
is the neighborhood of node $i$ in the graph.
In light of its connection with GCN, GAT can be studied as an object to investigate the \textbf{attention mechanism} under the controlled variable of GCN. We conduct the experiments in Section~\ref{sec:ana} and our findings are presented in Section~\ref{sec:att}.

\textbf{SGC}~\citep{wu2019simplifying}.
The Simplifying Graph Convolutional Network (SGC) is a decoupled GNN model, where ``\textbf{decoupled}" refers to a GNN model that separates the neural network transformation operators from the propagation (a.k.a, aggregation) operators~\cite{dong2021equivalence}.
Formally, SGC can be defined as: 
\begin{equation}
\label{equ:SGC}
\mathbf{Z}=\hat{\mathbf{A}}^K \mathbf{X} \mathbf{W},
\end{equation}
where $\hat{\mathbf{A}}^K$ can be seen as a composition of $K$ propagation layers, and $\mathbf{W}$ is a simple transformation layer. SGC can serve as the experimental group to study the influence of the \textbf{decoupled architecture} on GCN performance, while a $K$ layer GCN acts as the control group. The results of the control experiment are presented in Section~\ref{sec:appnp}.

\textbf{APPNP}~\citep{gasteiger2018predict}. 
APPNP is another decoupled model, which starts with transformation and subsequently proceeds to propagation. It can be expressed as:
\begin{equation}
\begin{aligned}
\mathbf{Z}^{(0)} & =\mathbf{H}=f_\theta(\mathbf{X}), 
\mathbf{Z}^{(k+1)}  =(1-\beta) \hat{\mathbf{A}} \mathbf{Z}^{(k)}+\beta \mathbf{H}, \\
\mathbf{Z}^{(K)} & =\operatorname{softmax}\left((1-\beta) \hat{\mathbf{A}} \mathbf{Z}^{(K-1)}+\beta \mathbf{H}\right),
\end{aligned}
\end{equation}
where $f_\theta$ represents a composition of multiple transformation layers (i.e., an MLP), $\mathbf{H}$ is the node representation of MLP, and 
$\beta$ specifies the teleport probability of personalized PageRank propagation. Only when $\beta$ is set to $0$, the propagation is equivalent to GCN and SGC. Another noticeable distinction between APPNP and SGC is the order of propagation and transformation layers. Therefore, in Section~\ref{sec:ana}, we have two settings for APPNP, with and without controlling $\beta=0$. These settings aim to distinguish between the effects of \textbf{decoupled architecture} and teleport in propagation. Detailed analyses are provided in Section~\ref{sec:appnp}.

\section{Investigating OOD Generalization of GNN Architectures} \label{sec:ana}
Graph out-of-distribution (OOD) generalization has primarily been addressed through learning-strategy-based and data-based methods. While these model-agnostic approaches provide flexibility, the potential influence of backbone GNN architectures on OOD generalization remains less explored. To delve deeper into this, we initiated a systematic analysis of various GNN architectures in OOD scenarios. When designing the study, we note that GNN architectures comprise multiple optional components, such as attention mechanisms and decoupled propagation. To evaluate their individual impacts, we adopted the most classic GCN model as our baseline and embarked on a series of ablation studies. Our ablation studies focus on three primary modifications: (1) substituting graph convolution with an attention mechanism, (2) decoupling feature transformation and feature propagation, and (3) removing the GCN's linear prediction layer. By controlling external variables, we were able to distinguish the contributions of these individual factors. The following sections provide detailed settings of our experiments and the consequential findings regarding these architectural components.

\subsection{Experimental Setup}

\noindent
\textbf{Evaluation Metric.}  
Our study aims to assess the impact of various backbone architectures on graph OOD generalization. In previous literature, researchers primarily use OOD test performance as their main metric. Yet, models display performance variations both in in-distribution (IID) and OOD settings. This suggests that OOD test performance alone is insufficient for a holistic evaluation of OOD generalization. To address this, we adopt the IID/OOD generalization gap metric from the NLP and CV domains~\cite{hendrycks2020pretrained, zhang2022delving} for graph OOD analysis.
The IID/OOD generalization gap depicts the difference between IID and OOD performance, offering a measure of models' sensitivity and robustness to distribution shifts. It is  defined as: \underline{${\textbf{GAP} = \textbf{IID}_\textbf{test}-\textbf{OOD}_\textbf{test}}$},
where $\text{IID}_\text{test}$ and $\text{OOD}_\text{test}$ are the test performance on IID and OOD test datasets, respectively.


To rigorously determine the component's influence on graph OOD generalization, we performed a paired T-Test between GCN baseline and other models. Evaluating across 10 runs for each model on every dataset, a p-value $\textless 0.05$ indicates a significant difference, with the t-value highlighting the superior model.


\vskip 0.2em
\noindent
\textbf{Dataset.}
We utilized the GOOD benchmark~\cite{gui2022good} for evaluating graph OOD generalization on node classification task. This benchmark offers a unique capability to simultaneously measure both IID and OOD performance, which is a feature not present in prevalent datasets from other studies, e.g., EERM~\cite{eerm}. Such simultaneous measurement is vital as calculating the GAP depends on both performance. The GOOD benchmark comprises citation networks, gamer networks, university webpage networks, and synthetic datasets, and delineates shifts as either covariate or concept shifts. From this collection, we selected 11 datasets, excluding CBAS due to its limited node size and feature dimension, and WebKB-university-covariate as all models exhibited high variance.


\vskip 0.2em
\noindent
\textbf{Implementation of GNN models.}
First, we follow the implementation of GOOD benchmark and use GCN~\cite{kipf} as our baseline model, which concludes with a linear prediction layer. In the GOOD paper~\cite{gui2022good}, models and hyperparameters are selected based on an OOD validation set. In contrast, our study emphasizes the inherent robustness of backbone models to unanticipated distribution shifts common in real-world scenarios. Consequently, we determine optimal hyperparameters using the IID validation set (Appendix~\ref{sec:hypsec3}). 

Next, to establish a comparative analysis framework, we implement other GNN models: GCN--, GAT~\cite{velivckovic2017graph}, SGC~\cite{wu2019simplifying}, and APPNP~\cite{gasteiger2018predict}, as delineated in Section~\ref{sec:pregnn}. Each implementation ensures consistency by maintaining all factors constant, except for the specific design under consideration. For instance, while deploying SGC, we align its hidden dimensions and propagation layers with those of the GCN baseline, even adjusting the transformation layer size to make it slightly different from the original SGC (i.e., the linear transformation in Eqn.~\ref{equ:SGC} is replaced by an MLP). This adjustment isolates the ``decoupled architecture" as the sole variation between SGC and GCN. By comparing these models with GCN, we discern how their distinct architectures influence graph OOD generalization.

In the following, we examine the impact of common GNN building blocks, i.e., \textbf{attention mechanism}, \textbf{coupled/decoupled architecture}, and \textbf{linear prediction layer}, respectively.

\subsection{Impact of Attention Mechanism} \label{sec:att}

\begin{table*}[t]
\caption{OOD and GAP performances for investigating the impact of self-attention.
All numerical results are averages across 10 random runs. \textcolor{red}{Red} color indicates the statistically significant improvement(i.e., $P_{value} \textless 0.05$ and $T_{value} \textgreater 0$) over the GCN. \textcolor{blue}{Blue} color indicates the statistically significant worse(i.e., $P_{value} \textless 0.05$ and $T_{value} \textless 0$) over the GCN. The best performance in each dataset is
highlighted in bold. OOD indicates the OOD performance on OOD test data.}
\vskip -1em
\label{table:GAT}
\resizebox{0.9\textwidth}{!}{%
\begin{tabular}{@{}cc|cc|cc|cc|cc|cc|cc@{}}
\toprule
\multicolumn{1}{l}{} &  & \multicolumn{2}{c}{G-Cora-Word} & \multicolumn{2}{c}{G-Cora-Degree} & \multicolumn{2}{c}{G-Arxiv-Time} & \multicolumn{2}{c}{G-Arxiv-Degree} & \multicolumn{2}{c}{G-Twitch-Language} & \multicolumn{2}{c}{G-WebKB-University} \\ \cmidrule(l){3-14} 
\multicolumn{1}{l}{} &  & OOD$\uparrow$ & GAP$\downarrow$ & OOD$\uparrow$ & GAP$\downarrow$ & OOD$\uparrow$ & GAP$\downarrow$ & OOD$\uparrow$ & GAP$\downarrow$ & OOD$\uparrow$ & GAP$\downarrow$ & OOD$\uparrow$ & GAP$\downarrow$ \\ \midrule
\multirow{2}{*}{Covariate} & GCN & {65.85} & \textbf{5.62} & 56.05 & 18.26 & 70.38 & 2.90 & 59.05 & 19.01 & 52.32 & 22.13 & - & -  \\  
 &\cellcolor{Gray}GAT & \cellcolor{Gray}\textbf{66.23} &\cellcolor{Gray}5.77 &\cellcolor{Gray}\textbf{56.12} &\cellcolor{Gray}\textbf{17.99} &\cellcolor{Gray}\textcolor{red}{\textbf{70.95}} &\cellcolor{Gray}\textbf{2.08} &\cellcolor{Gray}\textcolor{red}{\textbf{59.32}} &\cellcolor{Gray}\textbf{18.78} &\cellcolor{Gray}\textbf{52.83} &\cellcolor{Gray}\textbf{21.17} &\cellcolor{Gray}- &\cellcolor{Gray}- \\ \midrule
\multirow{2}{*}{Concept} & GCN & 65.44 & 3.05 & 62.48 & 7.05 & 62.50 & 13.42 & 60.13 & 16.36 & \textbf{45.12} & \textbf{39.25} & 26.42 & 39.41 \\ 
& \cellcolor{Gray}GAT & \cellcolor{Gray}\textcolor{red}{\textbf{65.86}} &\cellcolor{Gray}\textbf{2.21} &\cellcolor{Gray}\textcolor{red}{\textbf{63.85}} &\cellcolor{Gray}\textbf{4.83} &\cellcolor{Gray}\textcolor{red}{\textbf{64.96}} &\cellcolor{Gray}\textbf{10.89} &\cellcolor{Gray}\textcolor{red}{\textbf{63.07}} &\cellcolor{Gray}\textbf{11.99} &\cellcolor{Gray}\textcolor{blue}{44.14} &\cellcolor{Gray}40.81 &\cellcolor{Gray}\textcolor{red}{\textbf{29.27}} &\cellcolor{Gray}\textbf{35.23} \\ \bottomrule
\end{tabular}%
}
\end{table*}

In Table~\ref{table:GAT}, we assess the impact of attention mechanism by comparing OOD performance of GAT against GCN. GAT surpasses GCN on 10 out of 11 datasets, and shows a lower GAP value on 9 out of 11 datasets. Particularly, on the Arxiv-time-degree dataset, GAT improves OOD performance by $5.1\%$ relative to GCN and decreases the GAP by $36.4\%$, \textbf{underscoring the advantages of attention mechanism for graph OOD generalization}.

To statistically validate these observations, we further applied T-Tests to the OOD results of both models on each dataset. In Table~\ref{table:GAT}, red numbers denote that GAT significantly outperforms GCN, while blue signifies the opposite. The data reveals GAT's significant advantage on 7 datasets, further emphasizing the potency of attention in graph OOD generalization.  The detailed T-Test results are reported in Figure~\ref{fig:GAT_GCN}.(a) in Appendix~\ref{app:t-test}.


\vskip 0.2em
\noindent\textbf{Theoretical Insights.} Next, we present a theoretical analysis that delves into the success of GAT, elucidating why graph attention yields advantages for OOD generalization. Our analysis comprises two key components: (1) We establish a compelling link between the graph attention mechanism and the fundamental concept of information bottleneck; and (2) We demonstrate that optimizing the information bottleneck can notably enhance OOD generalization capabilities. At the start of our analysis, we introduce the concept of the information bottleneck (IB).



We denote the variables of input node features as $X$, and the variables of the output representations as $Z$. Thus, the mapping function of GNN $f_\theta(\cdot)$ can be expressed as $f(Z\mid X)$. Consider a distribution $X\sim \mathcal{N}(X^\prime,\epsilon)$ with $X$ as the noisy input variable, $X^\prime$ as the clean target variable, and $\epsilon$ as the variance for the Gaussian noise. Following \citet{kirsch2020unpacking}, the information bottleneck principle involves minimizing the mutual information between the input $X$ and its latent representation $Z$ while still accurately predicting $X^\prime$ from $Z$, and can be formulated as: 
\begin{equation}
f_{\mathrm{IB}}^*(Z \mid X)=\arg \min\nolimits_{f(Z \mid X)} I(X, Z)-I\left(Z, X^{\prime}\right)
\label{eq:ib}
\end{equation}
where ${I}(\cdot, \cdot)$ stands for the mutual information. With the aforementioned notations and concepts, we now introduce our proposition.
\begin{proposition}\label{pro: IB}
Given a node $i$ with its feature vector $x_i$ and its neighborhood $\mathcal{N}(i)$, the following aggregation scheme for obtaining its hidden representation ${\bf z}_i$, \begin{equation}\label{eqn: SA2IB}
\mathbf{z}_i = \sum_{j\in\mathcal{N}(i)} \frac{\eta_i \exp([\mathbf{W}_K \mathbf{x}_i]^\top \mathbf{W}_Q \mathbf{x}_j)}{\sum_{j\in\mathcal{N}(i)} \exp([\mathbf{W}_K \mathbf{x}_i]^\top \mathbf{W}_Q \mathbf{x}_j)}\mathbf{x}_j,
\end{equation}
with $\eta_i, \mathbf{W}_Q, \mathbf{W}_K$ being the learnable parameters, 
can be understood as the iterative process
to optimize the objective in Eq.~\eqref{eq:ib}. 
\end{proposition}

The detailed proof of the proposition mentioned above is available in Appendix~\ref{app:theory}. Proposition 1 unveils an intriguing relationship between the aggregation scheme in Eq.~\eqref{eqn: SA2IB} and the information bottleneck principle in Eq.~\eqref{eq:ib}: it demonstrates that the information bottleneck principle can be approached by adaptively aggregating similar node features into a learnable representation vector. It is worth noting that the aggregation scheme employed in Graph Attention Networks (GAT) (Eq.~\eqref{eqn: GAT}) can be viewed as a specific instance of Eq.~\eqref{eqn: SA2IB} under certain conditions: (1) GAT sets $\eta_i$ to a constant value; and (2) GAT computes attention using a weight matrix multiplication on the concatenated node pair vector. Given the insights provided by the proposition, it is reasonable to conjecture that GAT shares a similar connection with the information bottleneck principle.

Furthermore, we highlight that the information bottleneck principle plays a pivotal role in enhancing the OOD generalization of neural networks. Notably, \citet{ahuja2021invariance} have substantiated that a form of the information bottleneck constraint effectively addresses critical issues when invariant features completely capture the information about the label and also when they do not, under the context of distribution shifts. Consequently, we postulate that the reason the graph attention mechanism contributes to the OOD generalization of GNNs is intricately tied to its connection with the information bottleneck principle.

\subsection{Impact of Coupled/Decoupled Structure} 
\label{sec:appnp}

\begin{table*}[]
\caption{OOD and GAP performances for investigating the impact of decoupled architecture.
All numerical results are averages across 10 random runs. \textcolor{red}{Red} color indicates the statistically significant improvement(i.e., $P_{value} \textless 0.05$ and $T_{value} \textgreater 0$) over the GCN. \textcolor{blue}{Blue} color indicates the statistically significant worse(i.e., $P_{value} \textless 0.05$ and $T_{value} \textless 0$) over the GCN. The best performance in each dataset is
highlighted in bold.}
\vskip -1em
\label{table:APPNP}
\resizebox{0.9\textwidth}{!}{%
\begin{tabular}{@{}clcc|cc|cc|cc|cc|cc@{}}
\toprule
\multicolumn{1}{l}{} &  & \multicolumn{2}{c}{G-Cora-Word} & \multicolumn{2}{c}{G-Cora-Degree} & \multicolumn{2}{c}{G-Arxiv-Time} & \multicolumn{2}{c}{G-Arxiv-Degree} & \multicolumn{2}{c}{G-Twitch-Language} & \multicolumn{2}{c}{G-WebKB-University} \\ \cmidrule(l){3-14} 
\multicolumn{1}{l}{} &  & OOD$\uparrow$ & GAP$\downarrow$ & OOD$\uparrow$ & GAP$\downarrow$ & OOD$\uparrow$ & GAP$\downarrow$ & OOD$\uparrow$ & GAP$\downarrow$ & OOD$\uparrow$ & GAP$\downarrow$ & OOD$\uparrow$& GAP$\downarrow$\\ \midrule
\multirow{4}{*}{covariate} & GCN & 65.85 & 5.62 & 56.05 & 18.26 & 70.38 & 2.90 & 59.05 & 19.01 & 52.32 & 22.13 & - & - \\ 
 & SGC & 66.19 & 5.61 & 55.43 & 18.46 & 70.54 & \textbf{2.27} & \textcolor{red}{\textbf{59.66}} & \textbf{18.10} & \textcolor{red}{54.03} & 20.40 & - & - \\ 
 & \cellcolor{Gray}APPNP($\beta=0$) &\cellcolor{Gray}\textcolor{red}{66.66} &\cellcolor{Gray}\textbf{5.16} &\cellcolor{Gray}56.14 &\cellcolor{Gray}17.83 &\cellcolor{Gray}\textbf{70.69} &\cellcolor{Gray}2.84 &\cellcolor{Gray}\textcolor{red}{59.33} &\cellcolor{Gray}18.67 &\cellcolor{Gray}\textcolor{blue}{51.66} &\cellcolor{Gray}22.18 &\cellcolor{Gray}- &\cellcolor{Gray}- \\ 
 & \cellcolor{Gray}APPNP &\cellcolor{Gray}\textcolor{red}{\textbf{67.48}} &\cellcolor{Gray}6.37 &\cellcolor{Gray}\textcolor{red}{\textbf{58.33}} &\cellcolor{Gray}\textbf{17.16} &\cellcolor{Gray}\textcolor{blue}{69.22} &\cellcolor{Gray}3.38 &\cellcolor{Gray}\textcolor{blue}{55.40} &\cellcolor{Gray}22.10 &\cellcolor{Gray}\textcolor{red}{\textbf{56.47}} &\cellcolor{Gray}\textbf{16.75} &\cellcolor{Gray}- &\cellcolor{Gray}- \\ \midrule
\multirow{4}{*}{concept} & GCN & 65.44 & 3.05 & 62.48 & 7.05 & 62.50 & 13.42 & 60.13 & 16.36 & 45.12 & 39.25 & 26.42 & 39.41 \\ 
 & SGC & 65.21 & 3.16 & 62.41 & 6.77 & \textcolor{red}{63.88} & 11.94 & \textcolor{blue}{56.66} & 20.26 & \textcolor{blue}{44.22} & 40.52 & 25.78 & 40.38 \\ 
 & \cellcolor{Gray}APPNP($\beta=0$) &\cellcolor{Gray}\textcolor{red}{66.21} &\cellcolor{Gray}\textbf{2.81} &\cellcolor{Gray}\textcolor{red}{64.09} &\cellcolor{Gray}\textbf{4.85} &\cellcolor{Gray}\textcolor{red}{\textbf{65.10}} &\cellcolor{Gray}\textbf{10.94} &\cellcolor{Gray}\textcolor{red}{\textbf{65.44}} &\cellcolor{Gray}9.86 &\cellcolor{Gray}\textcolor{blue}{44.10} &\cellcolor{Gray}39.87 &\cellcolor{Gray}\textcolor{red}{29.17}&\cellcolor{Gray}35.83 \\ 
& \cellcolor{Gray}APPNP &\cellcolor{Gray}\textcolor{red}{\textbf{66.46}} &\cellcolor{Gray}4.24 &\cellcolor{Gray}\textcolor{red}{\textbf{64.81}} &\cellcolor{Gray}5.55 &\cellcolor{Gray}\textcolor{red}{63.46} &\cellcolor{Gray}11.51 &\cellcolor{Gray}\textcolor{red}{64.28} &\cellcolor{Gray}\textbf{9.51} &\cellcolor{Gray}\textcolor{red}{\textbf{46.81}} &\cellcolor{Gray}\textbf{35.99} &\cellcolor{Gray}\textcolor{red}{\textbf{30.28}} &\cellcolor{Gray}\textbf{35.22} \\ \bottomrule
\end{tabular}%
}
\end{table*}
\begin{figure*}[t]
    \begin{subfigure}{0.5\textwidth}
        \centering
        \includegraphics[width=\linewidth]{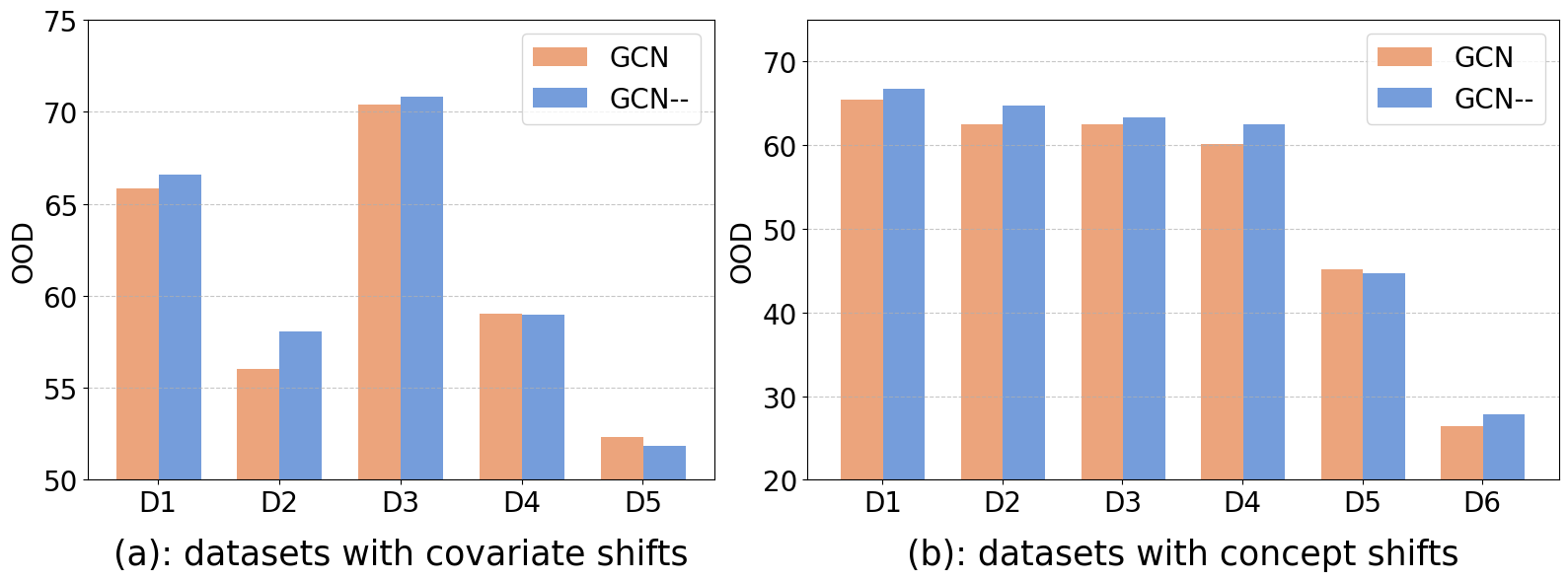}
    \end{subfigure}\hfill
    \begin{subfigure}{0.5\textwidth}
        \centering
        \includegraphics[width=\linewidth]{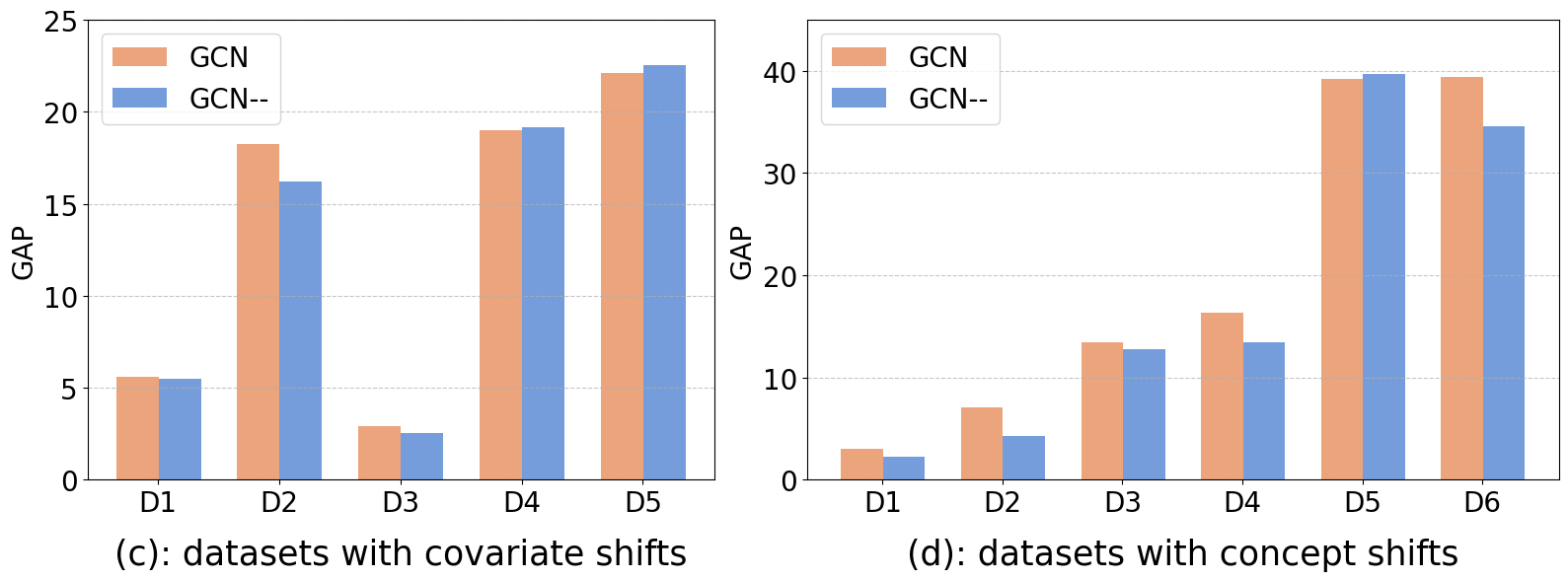}
    \end{subfigure}\hfill
    \vskip -1.2em
    \caption{Comparision of OOD and GAP between GCN and GCN-- for investagating the impact of linear classifier. GCN-- means GCN without linear classifier. D1, D2, D3, D4, D5, D6 represent G-Cora-Word, G-Cora-Degree, G-Arxiv-Time, G-Arxiv-Degree, G-Twitch-Language and G-WebKB-University respectively.}
    \label{fig:GCN--}
    \vspace{-0.25cm}
\end{figure*}


In order to compare coupled and decoupled structures, we evaluate the performance of various decoupled GNNs, presented in Table~\ref{table:APPNP}. 
For both OOD tests and GAP values, SGC surpasses GCN on merely 5 out of 11 datasets. In contrast, when the order of propagation and transformation is reversed, APPNP ($\beta = 0$) exceeds GCN's performance on 9 out of 11 datasets and demonstrates a lower GAP value on 9 out of 11 datasets. This reveals the significance of the transformation-propagation order in OOD generalization, \textbf{suggesting a preference for transformation prior to propagation in graph OOD contexts}. The T-Test results are presented in the same color scheme as in Section~\ref{sec:att}, which further confirms our initial observation.

Such observation can be potentially explained from the theory that decoupled graph neural networks are equivalent to label propagation, proposed by~\citet{dong2021equivalence}. From a label propagation perspective, the models propagate known labels across the graph to generate pseudo-labels for unlabeled nodes. These pseudo-labels then optimize the model predictor. The augmentation with pseudo-labels may curb overfitting and the architecture's training approach can adaptively assign structure-aware weights to these pseudo-labels~\cite{dong2021equivalence}. This might account for the enhanced OOD generalization performance seen in the decoupled architecture.

\subsection{Impact of Linear Prediction Layer} \label{sec:lin}

Lastly, we evaluate the impact of the linear prediction layer on a GCN. As shown in Figure~\ref{fig:GCN--}, GCN-- surpasses GCN on 8 out of 11 OOD datasets and achieves a lower GAP value on 8 out of 11 datasets. This indicates that \textbf{removing the last linear prediction layer can enhance graph OOD performance}. The T-Test results are reported in Figure~\ref{fig:GAT_GCN}.(b) in Appendix~\ref{app:t-test}. 

The performance dip of the linear prediction layer might be attributed to two factors. 
First, introducing an extra linear prediction layer might lead to surplus parameters and higher model complexity, amplifying the overfitting risk on IID. 
Second, the propagation process is non-parametric and has a lower risk of overfitting. It may be advantageous in both IID and OOD contexts. In contrast, the additional linear prediction layer is fit to the training data's label distribution, potentially hindering performance when faced with OOD distribution shifts. By omitting this layer, we amplify the graph's intrinsic structure, leading to improved OOD generalization.

\section{New GNN Design for Enhanced OOD Generalization}
\begin{figure*}     \centering     
\setlength{\abovecaptionskip}{0.cm}     \setlength{\belowcaptionskip}{0.cm}     \includegraphics[scale=0.9]{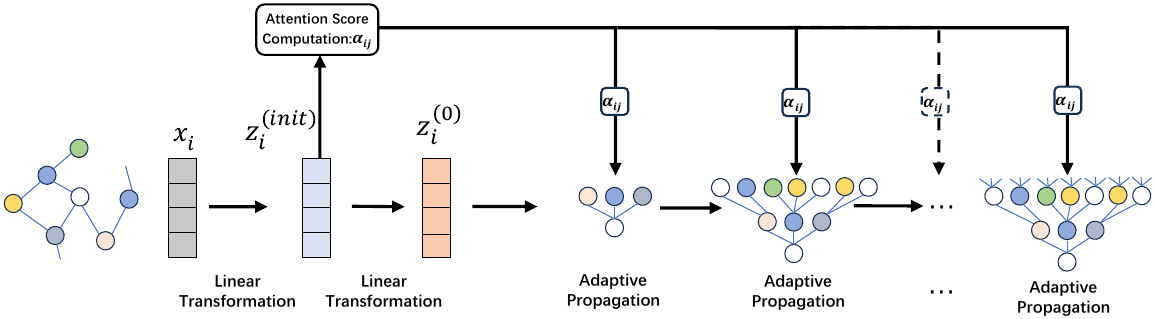}  
\captionsetup{skip=10pt}
\vskip -1em
\caption{An illustration of our proposed model \method{}. In this decoupled architecture, we calculate attention scores from transformed features and employ these scores throughout each propagation layer.}   \label{fig: model} 
\vskip -0.4em
\end{figure*}

In the previous section, we provided both \underline{empirical results} and \underline{theoretical analysis} that show the beneficial roles of the attention mechanism and the decoupled architecture in enhancing the OOD generalization of GNNs. On the other hand, we noted that the linear prediction layer detracts from OOD generalization. \textbf{Motivated by these observations, we propose to merge the attention mechanism with the decoupled architecture, opting to omit the linear prediction layer.} Specifically, we calculate attention scores from transformed features and employ these scores throughout each propagation layer. In the following, we delve into the details of our proposed model, Decoupled Graph Attention Network (\method).


\subsection{A New GNN Design}

Our proposed \method model integrates the principles of decoupled architecture and the attention mechanism, both of which have demonstrated positive contributions to OOD performance in previous sections. To adopt a decoupled architecture, we separate the linear transformation and propagation operations and conduct the linear transformation prior to propagation. Furthermore, we inject the attention mechanism into the propagation operations. As a result, our model includes three components: linear transformation, attention score computation, and adaptive propagation. The framework is illustrated by Figure~\ref{fig: model}.

\vskip 0.3em
\noindent
\textbf{Linear Transformation.} Our \method model decouples the GNN into transformation and propagation. This design can enhance OOD performance as discussed in Section~\ref{sec:appnp}. Therefore, \method first applies two linear transformation layers to the input data.
\begin{equation}
\begin{aligned}
\mathbf{Z}^{(\text{init})} & = \sigma{(\bm{W}^{(\text{init})}  \mathbf{X} + \bm{b}^{(\text{init})})} \quad \in \mathbb{R}^{N \times d}\\
\mathbf{Z}^{(0)} &  = \bm{H} = \bm{W}^{(0)}  \mathbf{Z}^{(\text{init})} + \bm{b}^{(0)} \quad \in \mathbb{R}^{N \times c}
\end{aligned}
\end{equation}
where $d$ is the hidden dimension and $c$ indicates the number of classes. The second linear layer maps features to the label space. These pseudo-labels help mitigate overfitting and dynamically assign structure-aware weights to them. 

\vskip 0.3em
\noindent
\textbf{Attention Score Computation.}
Our experimental findings in Section~\ref{sec:att} highlight the advantages of the attention mechanism for handling graph OOD, supported by the theoretical evidence that the graph self-attention mechanism aligns with the information bottleneck principle. 
To incorporate the attention mechanism into our model, we calculate attention scores based on the node representation $\mathbf{Z}^{(\text{init})}$, formulated as:
\begin{equation}
\alpha_{i j}=\frac{\exp \left(\operatorname{LeakyReLU}\left(\mathbf{a}^T\left[\mathbf{W} \mathbf{{Z}}^{(\text{init})}_i \| \mathbf{W} \mathbf{{Z}}^{(\text{init})}_j\right]\right)\right)}{\sum_{k \in \mathcal{N}_i} \exp \left(\operatorname{LeakyReLU}\left(\mathbf{a}^T\left[\mathbf{W} \mathbf{{Z}}^{(\text{init})}_i \| \mathbf{W} \mathbf{{Z}}^{(\text{init})}_k\right]\right)\right)},
\end{equation}
where $\alpha_{ij}$ represents the attention score. We use $\bm{P}$ to denote the attention matrix, where $\bm{P}_{i j}=\alpha_{ij}$.  We compute $\alpha_{ij}$ using $\mathbf{{Z}}^{(\text{init})}$ instead of $\mathbf{{Z}}^{(\text{0})}$, because $\mathbf{{Z}}^{(\text{0})}$ usually has a low dimensionality which can hamper the identification of important nodes. 

\vskip 0.3em
\noindent
\textbf{Adaptive Propagation.} The last step in \method is propagation.
Instead of a traditional fixed propagation, \method achieves adaptive propagation by combining attention score matrices and adjacency matrices. We define $\hat{\tilde{\mathbf{A}}}:=(1-\gamma)\bm{P}+\gamma \hat{\bm{A}}$. The adaptive propagation function is expressed as follows:
\begin{equation}
\begin{aligned}
\mathbf{Z}^{(k+1)} & =(1-\beta) \hat{\tilde{\mathbf{A}}} \mathbf{Z}^{(k)}+\beta \mathbf{H}, \\
\mathbf{Z}^{(K)} & =\operatorname{softmax}\left((1-\beta) \hat{\tilde{\mathbf{A}}} \mathbf{Z}^{(K-1)}+\beta \mathbf{H}\right),
\end{aligned}
\end{equation}
where $\beta$ is a hyperparameter to control the trade-off of the initial connection. 

Note that we do not employ a linear layer at the end of the model architecture as it can negatively impact the OOD generalization (Section~\ref{sec:lin}). As a consequence, our model combines both the decoupled architecture and attention mechanism, creating an elegant fusion of the strengths of the two preceding components.

\vskip 0.3em
\noindent
\textbf{Complexity analysis.} In the following, we will show that the computational complexity of our model is at the same level as GCN. We define the adjacency matrix as $\mathbf{A} \in \mathbb{R}^{N \times N}$, the input as $\mathbf{X} \in \mathbb{R}^{N \times d}$, and the transformation matrix as $\mathbf{W} \in \mathbb{R}^{d \times d}$. The operations of the GCN layer result in the following time complexity: $O\left(N d^2+N^2 d\right)$. Again noting that each multiplication with $\mathbf{A}$ is a sparse multiplication, we have $O\left(L|E| d^2+L N^2 d\right)$. $|E|$ represents the number of edges, $L$ represents the number of layers. The APPNP has the same computational complexity as GCN.

For self-attention, we need to compute $\mathbf{X}\mathbf{W}_Q$, $\mathbf{X} \mathbf{W}_K$, and $\mathbf{X} \mathbf{W}_V$, which each take $\mathrm{O}\left(N d^2\right)$ time. Unlike in the GCN case, we also now need to compute $\mathbf{Q} \mathbf{K}^{\top}$ in order to get the attention score $\boldsymbol{\alpha}$. This operation takes $\mathrm{O}\left(N^2 d\right)$ time. Finally, computing $\alpha \mathbf{V}$ takes $\mathrm{O}\left(N^2 d\right)$ time. These computations result in a time complexity of $\mathrm{O}\left(N^2 d+N d^2\right)$. All of these operations are computed at each layer, leading to the final time complexity of $\mathrm{O}\left(L N^2 d+L N d^2\right)$. Hence, the computational complexity of our model is 
$O(L|E| d^2+L N^2 d)$, which is at the same level of complexity as GCN.

\vskip 0.3em
\noindent\textbf{Advantages.} Despite its simplicity, our model stands out by offering several compelling advantages:

(a) \textit{Simple yet robust.} \method{} is grounded in the findings of our prior experimental study and theoretical analysis outlined in Section~\ref{sec:ana}. It enjoys the strengths of essential components that positively contribute to OOD generalization. 

(b) \textit{Favorable computational efficiency.} The efficiency analysis above reveals that the computational complexity of \method is comparable to that of traditional GNNs, yet it demonstrates superior OOD generalization.

(c) \textit{Compatible with diverse training strategies.} \method{} is orthogonal to external OOD techniques and can achieve further OOD generalization from OOD training strategies. In the following sections, we will empirically verify that this model can function as a powerful backbone for various popular OOD algorithms.

\subsection{Experiment} 
\begin{table*}[]
\caption{OOD and GAP performances under ERM setting on datasets from GOOD.
All results are averages over 10 random runs.}
\vskip -1em
\label{table:ERM}
\resizebox{0.95\textwidth}{!}{%
\begin{tabular}{@{}cll|cc|cc|cc|cc|cc|cc@{}}
\toprule
\multicolumn{1}{l}{} &  &  &\multicolumn{2}{c}{G-Cora-Word} & \multicolumn{2}{c}{G-Cora-Degree} & \multicolumn{2}{c}{G-Arxiv-Time} & \multicolumn{2}{c}{G-Arxiv-Degree} & \multicolumn{2}{c}{G-Twitch-Language} & \multicolumn{2}{c}{G-WebKB-University}
\\ \cmidrule(l){2-15} 
\multicolumn{1}{c}{} & & & OOD$\uparrow$ & GAP$\downarrow$ & OOD$\uparrow$ & GAP$\downarrow$ & OOD$\uparrow$ & GAP$\downarrow$ & OOD$\uparrow$ & GAP$\downarrow$ & OOD$\uparrow$ & GAP$\downarrow$ & OOD$\uparrow$ & GAP$\downarrow$ \\ \midrule
\multirow{7}{*}{Covariate} & GCN-- &  & 66.56 & \textbf{5.47} & 58.09 & 16.22 & 70.82 & 2.52 & 58.95 & 19.15 & 51.86 & 22.55 &- & -  \\ 
 & SGC & \textbf{} & 66.23 & 5.51 & 55.86 & 17.98 & 70.27 & 1.87 & 59.36 & 17.96 & 53.19 & 16.26 & - & - \\ 
 & APPNP &  & 67.62 & 6.28 & 58.87 & 16.66 & 69.2 & 2.43 & 56.25 & 20.74 & 56.85 & 16.67 & - & - \\ 
 & GAT & \textbf{} & 65.84 & 5.62 & 58.25 & 15.66 & 70.65 & 2.38 & 58.45 & 19.78 & 53.3 & 20.63 & \textbf{-} & - \\ 
 & GraphSAGE &  & 65.59 & 7.67 & 56.54 & 19.24 & 69.36 & 3.00 & 57.59 & 19.67 & 55.88 & 17.75 & - & - \\ 
 & GPRGNN &  & 67.59 & 6.44 & 59.46 & 15.75 & 67.74 & 2.66 & 56.68 & 19.02 & 56.37 & 16.05 & - & - \\  
 & \cellcolor{Gray}\method{} &\cellcolor{Gray}  &\cellcolor{Gray}\textbf{67.67} &\cellcolor{Gray}6.09 &\cellcolor{Gray}\textbf{59.68} &\cellcolor{Gray}\textbf{15.02} &\cellcolor{Gray}\textbf{71.33} &\cellcolor{Gray}\textbf{1.65} &\cellcolor{Gray}\textbf{60.12} &\cellcolor{Gray}\textbf{17.56} &\cellcolor{Gray}\textbf{57.37} &\cellcolor{Gray}\textbf{15.31} &\cellcolor{Gray}- &\cellcolor{Gray}- \\ \midrule
\multirow{7}{*}{Concept} & GCN-- &  & 66.70 & 2.31 & 64.72 & 4.27 & 63.27 & 12.78 & 62.46 & 13.47 & 44.72 & 39.65 & 27.80 & 34.54 \\ 
 & SGC &  & 66.28 & \textbf{1.96} & 62.58 & 7.12 & 63.33 & 11.61 & 55.06 & 21.28 & 47.43 & 35.31 & 29.63 & 33.37 \\ 
 & APPNP &  & 67.31 & 4.12 & 66.3 & 4.83 & 63.64 & 10.76 & 63.92 & 8.83 & 48.1 & 34.08 & 26.88 & 44.45 \\ 
 & GAT &  & 66.32 & 2.05 & 64.41 & \textbf{4.10} & 65.02 & 10.86 & 64.75 & 9.00 & 43.95 & 40.89 & 28.35 & \textbf{32.65} \\ 
 & GraphSAGE &  & 65.42 & 5.46 & 65.06 & 5.43 & 62.85 & 11.77 & 60.92 & 12.96 & 45.51 & 39.93 & \textbf{34.50} & 40.67 \\ 
 & GPRGNN &  & 66.95 & 3.59 & 65.97 & 4.96 & 61.89 & 10.96 & 63.05 & 8.06 & \textbf{49.07} & \textbf{33.08} & 27.06 & 40.27 \\  
& \cellcolor{Gray}\method{} &\cellcolor{Gray}  &\cellcolor{Gray}\textbf{67.50} &\cellcolor{Gray}3.57 &\cellcolor{Gray}\textbf{66.36} &\cellcolor{Gray}4.45 &\cellcolor{Gray}\textbf{65.11} &\cellcolor{Gray}\textbf{10.17} &\cellcolor{Gray}\textbf{65.86} &\cellcolor{Gray} \textbf{7.73} &\cellcolor{Gray}45.10 &\cellcolor{Gray}40.20 &\cellcolor{Gray}33.57 &\cellcolor{Gray}37.92 \\ \bottomrule
\end{tabular}
}
\end{table*}

\begin{table*}[ht]
\centering
\caption{OOD performances under ERM and EERM on datasets from EERM paper.
All results are averages over 5 random runs.}
\vskip -1em
\label{table:ERM1}
\scalebox{1.05}{
\small
\begin{tabular}{cccccccccc}
\hline
Dataset & 
Method & GCN-- & SGC & APPNP & GAT & GraphSAGE & GPRGNN & \textbf{\method} \\ \hline
Amz-Photo & 
ERM & 93.79±0.97 & 93.83±2.30 & 94.44±0.29 & 96.30±0.79 & 95.09±0.60 & 91.87±0.65 & \cellcolor{Gray}\textbf{96.56±0.85} \\
Cora &
ERM & 91.59±1.44 & 92.17±2.38 & 95.16±1.06 & 94.81±1.28 & 99.67±0.14 & 93.00±2.17 &\cellcolor{Gray}\textbf{99.68±0.06} \\
Elliptic &
ERM & 50.90±1.51 & 49.19±1.89 & 62.17±1.78 & 65.36±2.70 & 56.12±4.47 & 64.59±3.52 &\cellcolor{Gray}\textbf{73.09±2.14} \\
OGB-Arxiv &
ERM & 38.59±1.35 & 41.44±1.49 & 44.84±1.43 & 40.63±1.57 & 39.56±1.66 & 44.38±0.59 &\cellcolor{Gray}\textbf{45.95±0.65} \\
Twitch-E &
ERM & 59.89±0.50 & 59.61±0.68 & 61.05±0.89 & 58.53±1.00 & 62.06±0.09 & 59.72±0.40 & \cellcolor{Gray}\textbf{62.14±0.23} \\ \hline
Amz-Photo &
EERM & 94.05±0.40 & 92.21±1.10 & 92.47±1.04 & 95.57±1.32 & \textbf{95.57±0.13} & 90.78±0.52 & \cellcolor{Gray}92.54±0.77 \\
Cora &
EERM & 87.21±0.53 & 79.15±6.55 & 94.21±0.38 & 85.00±0.96 & 98.77±0.14 & 88.82±3.10 &\cellcolor{Gray}\textbf{98.85±0.26} \\
Elliptic &
EERM & 53.96±0.65 & 45.37±0.82 & 58.80±0.67 & 58.14±4.71 & 58.20±3.55 & 67.27±0.98 & \cellcolor{Gray}{\textbf{68.74±1.12}} \\
OGB-Arxiv &
EERM & OOM & OOM & OOM & OOM & OOM & OOM & \cellcolor{Gray}OOM \\
Twitch-E &
EERM & 59.85±0.85 & 54.48±3.07 & 62.28±0.14 & 59.84±0.71 & 62.11±0.12 & 61.57±0.12 &  \cellcolor{Gray}{\textbf{62.52±0.09}} \\ \hline
\end{tabular}%
}
\vspace{-0.2cm}
\end{table*}


To assess the OOD generalization capabilities of the proposed \method, we conduct experiments under various training strategies on node classification tasks. Through experiments, we aimed to answer the following questions: 
\textbf{Q1}: Can \method{} outperform existing GNN architectures on OOD test data? \textbf{Q2}: Is \method{} a better backbone model in different OOD generalization methods? 


\subsubsection{OOD performance of \method{}}
To answer \textbf{Q1}, we evaluate the performance of our proposed \method{} on ERM setting. 

\noindent
\textbf{Baselines.}
We evaluate the performance of our \method{} by comparing it with several state-of-the-art models, including GCN, APPNP, GAT, SGC, GPRGNN~\cite{chien2021adaptive}, and GraphSAGE~\cite{hamilton2017inductive}.

\noindent
\textbf{Datasets.}
In this experiment, we selected 11 datasets from the GOOD benchmark, the same as introduced in Section~\ref{sec:ana}.
In addition, we conducted experiments on 5 new datasets that are used in EERM~\cite{eerm} paper. 
These datasets exhibit diverse distribution shifts: Cora and Amz-Photo involve synthetic spurious features; Twitch-E exhibits cross-domain transfer with distinct domains for each graph; Elliptic and OGB-Arxiv represent temporal evolution datasets, showcasing dynamic changes and temporal distribution shifts. The details of these datasets are shown in Appendix~\ref{sec:data}

\noindent
\textbf{Implementation Details.}
In Section~\ref{sec:lin}, we observed that using linear transformation as the final prediction layer degrades the OOD performance. Consequently, we replace the linear prediction layer with an individual graph convolutional layer for all the models. Similar to our prior experiments, 
we train and select model hyperparameters on IID distribution and test the model on OOD distribution. 
The number of layers is chosen from $\{2,3\}$. The hidden size is chosen from $\{100,200,300\}$. We tune the following hyper-parameters: $\gamma \in\{0, 0.2, 0.5\}$, $\beta \in\{0, 0.1, 0.2, 0.5\}$. The details of parameters are shown in Appendix~\ref{sec:hypsec3}.

\noindent
\textbf{Experimental Results.}
The results on GOOD datasets are reported in Table~\ref{table:ERM}. From this table, we find that \method{} outperforms baselines on 9/11 datasets for the OOD
test. Meanwhile, \method{} exhibits a lower GAP value compared to baselines on 6 out of 11 datasets. For example, \method{} delivers an improvement of 1.7\% over baselines for OOD test while achieving a decline of 4.1\% over baselines for GAP on GOODArxiv-degree-concept, which indicates the effectiveness of our model for Graph OOD generalization.
The results on datasets from EERM paper under ERM setting are reported in Table~\ref{table:ERM1}. Remarkably, our proposed \method{} outperforms baseline GNN backbones on all of these datasets.

To further validate the contribution of each component in \method{} and the robustness of \method{} to the choice of hyperparameters (i.e., $\gamma$ and $\beta$), we conduct additional ablation study and hyperparameter study. The results confirm the effectiveness of our proposed method and are shown in Appendix~\ref{sec:abl}.




\subsubsection{\method{} Performance as a Backbone}
In order to answer the $\textbf{Q2}$, we conduct experiments to evaluate our model and baselines across various strategies proposed for OOD.
Specifically, we choose GCN-- and APPNP that perform well under ERM as the backbones. 
Representative OOD algorithms such as IRM~\cite{arjovsky2019invariant}, VREx~\cite{krueger2021out}, GroupDRO~\cite{sagawa2019distributionally}, Graph-Mixup~\cite{wang2021mixup} and EERM~\cite{eerm} are considered as baseline methods. Among them, Graph-Mixup and EERM are graph-specific methods.
It's worth noting that within the GOOD framework, Graph-Mixup is equipped with GCN incorporating a linear classifier, and Graph-Mixup is more suitable for this framework in APPNP, as it is better suited for enhancement at the hidden dimension level rather than the class dimension. Hence, for Graph-Mixup, we have added a linear classifier on top of the models. 

\begin{figure*}[t]
    \begin{subfigure}{0.5\textwidth}
        \centering
        \includegraphics[width=\linewidth]{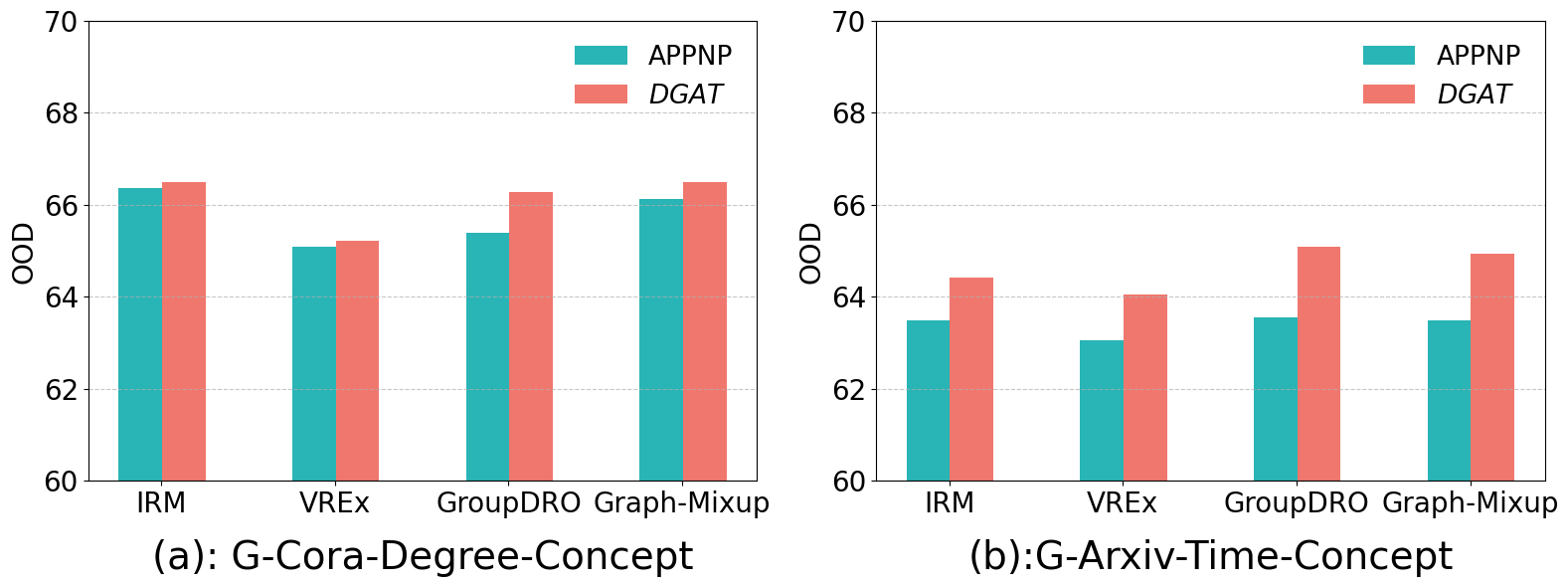}
    \end{subfigure}\hfill
    \begin{subfigure}{0.5\textwidth}
        \centering
    \includegraphics[width=\linewidth]{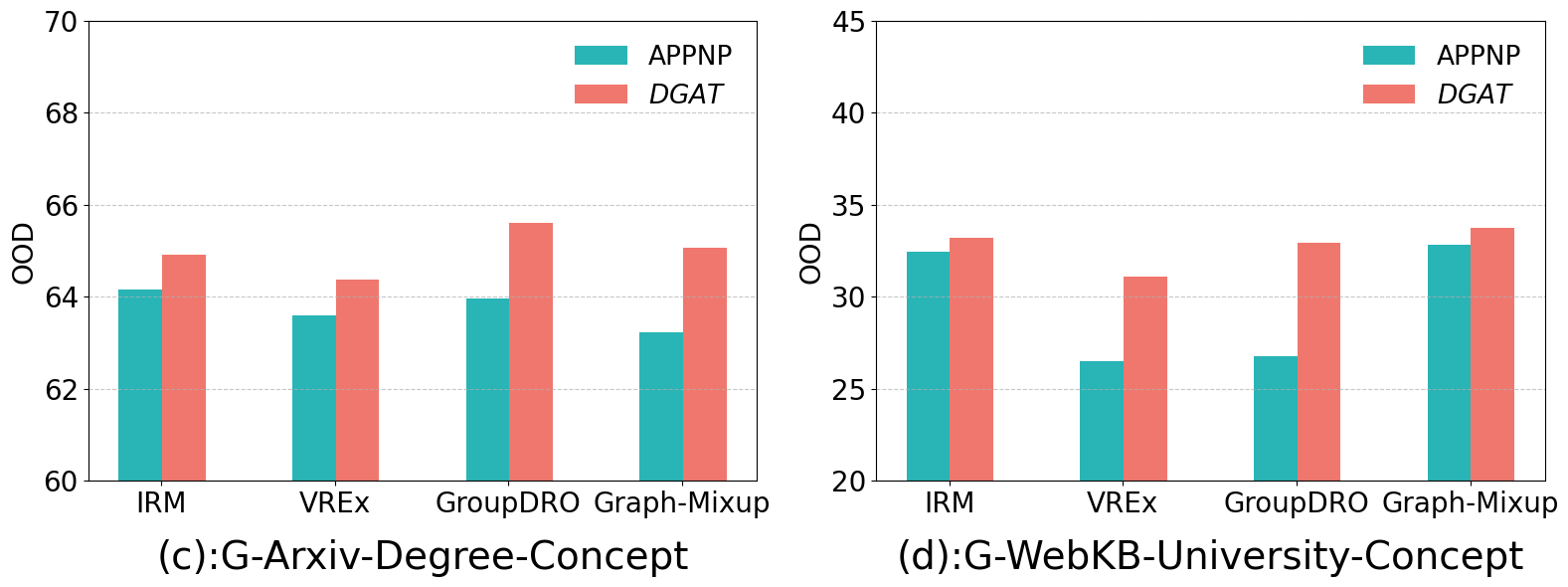}
    \end{subfigure}\hfill
    \vskip -1em
    \caption{Comparision of OOD performance between \method{} and APPNP equipped with various OOD algorithms. Results on more datasets are reported in Appendix~\ref{sec:backbone}.}
    \label{fig:OOD_method}
    \vspace{-0.4cm}
\end{figure*}
The experimental results illustrating the performances related to OOD and GAP across various training strategies (i.e, IRM, VRex, GroupDRO, Graph-Mixup) on datasets from GOOD are shown in Figure~\ref{fig:OOD_method}. The results of OOD performance under EERM setting on datasets from EERM are reported in Table~\ref{table:ERM1}. The results of OOD performance under EERM setting on datasets from GOOD benchmark are reported inTable~\ref{table:ERRMsetting1}, which is presented in Appendix~\ref{sec:backbone}. Results on more datasets across various training strategies are also included in Appendix~\ref{sec:backbone}. We have the following observations.

(a) First, \textbf{compared to the baselines, our \method model consistently demonstrates better OOD generalization when combined with external OOD algorithms.} We find that \method{} outperforms baselines on 9/11 datasets under IRM training strategy. For example, \method{} delivers an improvement of 3.6\% over baselines for OOD test on G-Cora-Degree-Covariate under IRM. Meanwhile, \method{} demonstrates superior OOD performance in comparison to baselines across 9 out of 11 datasets under Graph-Mixup training setting. 
In Table~\ref{table:ERM1} and Table~\ref{table:ERRMsetting1}, our model also achieves better performance on most datasets under EERM training setting. 

(b) Second, \textbf{OOD training algorithms do not always improve the OOD performance of backbone models}. For example, on G-Cora-Word and G-Cora-Degree datasets, all backbone models suffer from OOD performance degradation when trained with EERM algorithm. 
This observation is consistent with the results in GOOD paper~\cite{gui2022good} and highlights the limitation of existing OOD algorithms.

(c) Meanwhile, \textbf{we indeed notice a significant improvement in certain cases.}
For example, on G-Twitch-language-concept, \method{} equipped with VREx achieves an improvement of 7.0\% over ERM. This confirms that combining with an effective training algorithm can further enhance the OOD generalization of our \method backbone model.

\section{Related work}
\subsection{OOD generalization}
In the real world, when the distribution of training data differs from that of testing data, denoted as $P_{tr}(X,Y) \neq P_{te}(X,Y)$, this distribution shift is referred to as an OOD problem. Common types of  distribution shifts including covariate shift, concept shift, and prior shift~\cite{li2023graph}. 
To effectively achieve better OOD generalization, several methods have been proposed~\cite{arjovsky2019invariant,sagawa2019distributionally,krueger2021out,ganin2016domain,sun2016deep}.
For instance, IRM~\cite{arjovsky2019invariant} is a representative method that aims to discover invariant features in which the optimal classifier consistently performs across all environments. 
GroupDRO~\cite{sagawa2019distributionally} improves the model's out-of-distribution (OOD) generalization by optimizing for the worst-case scenario over a set of predefined groups and introducing strong regularization.
VREx~\cite{krueger2021out} considers there exists variation among different training domains, and this variation can be extrapolated to test domains. Based on this assumption, the method aims to make the risks across different training domains as consistent as possible, reducing the model's reliance on spurious features.  Notably, these techniques are primarily tailored for images and may not fully account for the unique attributes of graph data, thus maybe suboptimal for graph OOD generalization.

\subsection{OOD generalization on graphs}
In graph-structured data, the OOD problem exists as well, but the research on graph OOD is currently in its early stages.
The covariate shift and concept shift also exit in the graph domain. 
Unlike the general OOD problem, in graph-based OOD, shifts can occur not only in features but may also occur implicitly in the graph structure.
Some efforts have been proposed to solve the graph OOD problem in node classification tasks from two perspectives: data-based methods and learning-strategy-based methods. Data-based methods focus on manipulating the input graph data to boost OOD generalization~\cite{GTrans,wang2021mixup,li2023graph}. For example, GTrans~\cite{GTrans} provides a data-centric view to solve the graph OOD problem and propose to transform the graph data at test time to enhance the ability of graph OOD generalization.
On the other hand,
learning-strategy-based methods emphasize modifying training
approaches by introducing specialized optimization objectives and
constraints~\cite{eerm,Lisa, liu2023flood}. For example, 
EERM~\cite{eerm} seeks to leverage the invariant associations between features and labels across diverse distributions, thereby achieving demonstrably satisfactory OOD generalization in the face of distribution shifts. 
However, current methods predominantly emphasize external techniques to enhance OOD generalization, yet they do not provide insights into the inherent performance of the underlying backbone models themselves. Therefore, in this work, we investigate the impact of the GNN architecture on graph OOD. Building upon our discoveries, we introduce a novel model aimed at improving the OOD generalization ability on graphs. 

\section{Conclusion}
GNNs tend to yield suboptimal performance on out-of-distribution (OOD) data. While prior efforts have predominantly focused on enhancing graph OOD generalization through data-driven and strategy-based methods, relatively little attention has been devoted to assessing the influence of GNN backbone architectures on OOD generalization. To bridge this gap, we undertake the first comprehensive examination of the OOD generalization capabilities of well-known GNN architectures.
Our investigation unveils that both the attention mechanism and the decoupled architecture positively impact OOD generalization. Conversely, we observe that the linear classification layer tends to compromise OOD generalization ability. To deepen our insights, we provide theoretical analysis and discussions.
Building upon our findings, we introduce a novel GNN design, denoted as \method, which combines the self-attention mechanism and the decoupled architecture. 
Our comprehensive experiments across a variety of training strategies show that the GNN backbone architecture is indeed important, and that combining useful architectural components can lead to a superior GNN backbone architecture for OOD generalization.

\bibliographystyle{ACM-Reference-Format}
\bibliography{sample-base}

\newpage
\appendix
\section{Appendices}


\subsection{Results of T-Test}
\label{app:t-test}
To statistically validate these observations, we further applied
T-Tests to the OOD results of both models on each dataset. The results are shown in Figure~\ref{fig:GAT_GCN} and Figure~\ref{fig: SGC}.
\subsection{Proof of Proposition ~\ref{pro: IB}}
\label{app:theory}
\setcounter{proposition}{0}
\begin{proposition}
Given a node $i$ with its feature vector $x_i$ and its neighborhood $\mathcal{N}(i)$, the following aggregation scheme for obtaining its hidden representation ${\bf z}_i$, \begin{equation}
\mathbf{z}_i = \sum_{j\in\mathcal{N}(i)} \frac{\eta_i \exp([\mathbf{W}_K \mathbf{x}_i]^\top \mathbf{W}_Q \mathbf{x}_j)}{\sum_{j\in\mathcal{N}(i)} \exp([\mathbf{W}_K \mathbf{x}_i]^\top \mathbf{W}_Q \mathbf{x}_j)}\mathbf{x}_j, \nonumber
\end{equation}
with $\eta_i, \mathbf{W}_Q, \mathbf{W}_K$ being the learnable parameters, 
can be understood as an iterative process
to optimize the objective in Eq.~\eqref{eq:ib}.
\end{proposition}

\begin{proof}
Given a distribution $X\sim \mathcal{N}(X^\prime,\epsilon)$ with $X$ as the observed input variable and $X^\prime$ as the clean target variable. Following \cite{kirsch2020unpacking}, the information bottleneck principle involves minimizing the mutual information between the input $X$ and its latent representation $Z$ while still accurately predicting $X^\prime$ from $Z$. In the context of deep learning, the information bottleneck principle can be formulated as the following optimization objective:
\begin{equation}\label{eqn: MI}
f_{\mathrm{IB}}^*(Z \mid X)=\arg \min _{f(Z \mid X)} I(X, Z)-I\left(Z, X^{\prime}\right)
\end{equation}

As demonstrated by \citet{still2003geometric}, we can utilize the information bottleneck to solve the clustering problems, in which nodes will be clustered into clusters with indices $c$.  For simplicity, we take the 1-hop graph of node $u$ to illustrate where the node indices are $1, 2, \ldots, |\mathcal{N}(u)|$. 
Following~\citet{strouse2019information}, we assume that $p(i) = \frac{1}{n}$ with $n=|\mathcal{N}(u)|$ and $p(\mathbf{x}|i) \propto \exp{[-\frac{1}{2\epsilon^2}||\mathbf{x} - \mathbf{x}_i||^2]} $ with the introduction of a smoothing parameter $\epsilon$. 

We denote $p_t$ as the probability distribution after the $t$-th iteration, and the iterative equation is given by \citep{zhou2022understanding}:
\begin{align}
\label{eqn:iterative}
p_t(c|i) & = \frac{\log p_{t-1}(c)}{Z(i)} \exp \left[- D_\text{KL}\left[p(\mathbf{x}|i) | p_{t-1}(\mathbf{x}|c)\right]\right]  \\
p_{t}(c) & = \frac{n_{t}^{(c)}}{n} \\
p_{t}(\mathbf{x}|c) & = \frac{1}{n_{t}^{(c)}} \sum_{i \in S_{t}^{(c)}} p(\mathbf{x}|i),
\end{align}
where $Z(i)$ ensures normalization, $S_{t}^{(c)}$ represents the set of node indices in cluster $c$, and $n_t^{(c)}=|S_{t}^{(c)}|$ is the number of nodes assigned to cluster $c$. Then, we can approximate $p_{t-1}(\mathbf{x}|c)$ using a Gaussian distribution $q_{t-1}(\mathbf{x}|c)\sim \mathcal{N}(\mu_{t-1}^{(c)},\Sigma_{t-1}^{(c)})$. When the value of $\epsilon$ w.r.t. $p(\mathbf{x}|i)$ is sufficiently small, we have:
\begin{equation}\label{enq: KL}
\begin{split}
D_\text{KL}\left[p(\mathbf{x}|i) | q_{t-1}(\mathbf{x}|c)\right] \propto 
& [\mu_{t-1}^{(c)}-\mathbf{x}_i]^\top [\Sigma_{t-1}^{(c)}]^{-1}[\mu_{t-1}^{(c)}-\mathbf{x}_i] \\
& + \log \det\Sigma_{t-1}^{(c)} + B, 
\end{split}
\end{equation}
where $B$ represents terms that are independent of the assignment of data points to clusters and are consequently irrelevant to the objective. By substituting Eq.~\eqref{enq: KL} back into Eq.~\eqref{eqn:iterative}, we can reformulate the cluster update as follows:
\begin{align}
\begin{split}
p_t(c|i) & = \frac{\log p_{t-1}(c)}{Z(i)} \frac{\exp \left[- [\mu_{t-1}^{(c)}-\mathbf{x}_i]^\top [\Sigma_{t-1}^{(c)}]^{-1}[\mu_{t-1}^{(c)}-\mathbf{x}_i] \right]}{\det\Sigma_{t-1}^{(c)}}  \\
& = \frac{\log p_{t-1}(c)}{\det\Sigma_{t-1}^{(c)}} \frac{\exp \left[- [\mu_{t-1}^{(c)}-\mathbf{x}_i]^\top [\Sigma_{t-1}^{(c)}]^{-1}[\mu_{t-1}^{(c)}-\mathbf{x}_i] \right]}{\sum_c\exp \left[- [\mu_{t-1}^{(c)}-\mathbf{x}_i]^\top [\Sigma_{t-1}^{(c)}]^{-1}[\mu_{t-1}^{(c)}-\mathbf{x}_i] \right] }.
\end{split}
\end{align}
To minimize the KL-divergence between $p_{t-1}(\mathbf{x}|c)$ and $ q_{t-1}(\mathbf{x}|c)$, $\mu_c^{(t)}$ will be updated as:
\begin{align}
\begin{split}
    \mu_c^{(t)} & = \frac{1}{n} \sum_{i=1}^{n} p_t(c|i) \mathbf{x}_i \\
    & = \frac{1}{n} \sum_{i=1}^{n}  \frac{\log p_{t-1}(c)}{\det\Sigma_{t-1}^{(c)}} \frac{\exp \left[- [\mu_{t-1}^{(c)}-\mathbf{x}_i]^\top [\Sigma_{t-1}^{(c)}]^{-1}[\mu_{t-1}^{(c)}-\mathbf{x}_i] \right]}{\sum_c\exp \left[- [\mu_{t-1}^{(c)}-\mathbf{x}_i]^\top [\Sigma_{t-1}^{(c)}]^{-1}[\mu_{t-1}^{(c)}-\mathbf{x}_i] \right] } \mathbf{x}_i \\
    & =  \sum_{i=1}^{n}  \frac{\log p_{t-1}(c)}{n\det\Sigma_{t-1}}  \frac{\exp \left[2\cdot [\mu_{t-1}^{(c)}]^\top \Sigma_{t-1}^{-1}\mathbf{x}_i \right]}{\sum_c\exp \left[2\cdot [\mu_{t-1}^{(c)}]^\top \Sigma_{t-1}^{-1}\mathbf{x}_i \right]} \mathbf{x}_i,
\end{split}
\end{align}
where the last equation follows from the assumption that $\Sigma_{t-1}^{c}$ is shared among all clusters and $\mu_c$ are normalized w.r.t. $\Sigma_{t-1}^{-1}$. 
Let $\mathbf{z}_c = \mu_c^{(t)}$ ,$\eta_c= \frac{\log p_{t-1}(c)}{n\det\Sigma_{t-1}} $, $2\cdot\mu_{t-1}^{(c)} = \mathbf{W}_{K} \mathbf{x}_c$, $\mathbf{W}_{Q} = \Sigma_{t-1}^{-1}$ and rewrite the subscripts appropriately to obtain:
\begin{equation}
\mathbf{z}_i = \sum_{j\in\mathcal{N}(i)} \frac{\eta_i \exp([\mathbf{W}_K \mathbf{x}_i]^\top \mathbf{W}_Q \mathbf{x}_j)}{\sum_{j\in\mathcal{N}(i)} \exp([\mathbf{W}_K \mathbf{x}_i]^\top \mathbf{W}_Q \mathbf{x}_j)}\mathbf{x}_j, \nonumber
\end{equation}
where $\eta$ indicates attention correction weighting parameters,   $\mathbf{W}_Q$ and $\mathbf{W}_K$ are transformation parameter about input features. 
\end{proof}
This indicates that the graph self-attention mechanism follows the information bottleneck principle. Specifically, $\mu_c^{(t)}$ refers to the data distribution learned by the information bottleneck, and $\mathbf{z}_c$ is learned by the self-attention mechanism. 





\subsection{Dataset Details} \label{sec:data}
We majorly leverage the datasets from GOOD~\cite{gui2022good} to evaluate the model. For example, GOOD-Cora is a citation network derived from the complete Cora dataset. It involves a small-scale citation network graph where nodes correspond to scientific publications, and edges represent citation links. The objective is a 70-class classification of publication types. Splits are generated based on two domain selections: word and degree. GOOD-Arxiv is a citation dataset adapted from OGB. It features a directed graph representing the citation network among computer science (CS) arXiv papers. Nodes denote arXiv papers, and directed edges signify citations. The task involves predicting the subject area of arXiv CS papers, constituting a 40-class classification challenge. Splits are generated based on two domain selections: time (publication year) and node degree.

Additionally, we use datasets from EERM to evaluate our model. The statistic information are summarised in Table~\ref{tab:eermdata}. These datasets have different types of distribution shift. The shift types for Cora and Amz-Photo are labeled as Artificial
Transformation, Twitch-E is categorized as Cross-Domain Transfers, Elliptic and OGB-Arxiv are identified as Temporal Evolution.
\begin{table}[t]
\centering
\caption{Statistic information for datasets from EERM}
\begin{tabular}{ccccc}
\hline Dataset & \#Nodes & \#Edges & \#Classes \\
\hline Cora  & 2,703 & 5,278 & 10 \\
Amz-Photo & 7,650 & 119,081 & 10 \\
Twitch-E  & $1,912-9,498$ & $31,299-153,138$ & 2 \\
Elliptic & 203,769 & $2,34,355$ & 2 \\
OGB-Arxiv & 169,343 & $1,166,243$ & 40 \\
\hline
\end{tabular}
\label{tab:eermdata}
\end{table}\label{sec:eermdata1}


\subsection{Hyperparameter Selection} 
\begin{table}[t]
\centering
\caption{Parameter Searching Space}
\resizebox{0.4\textwidth}{!}{
\begin{tabular}{|c|c|}
\hline
learning rate (lr) & $[5e-3, 1e-3]$ \\

dropout & $[0.1, 0.2, 0.5]$ \\

hidden & $[100, 200, 300]$ \\

number of model layer & $[1, 2, 3]$ \\
\hline
\end{tabular}}
\label{tab:parameter}
\end{table}\label{sec:hypsec3}
In the investigating experiments, we perform a hyperparameter search to obtain experimental results that can reflect the generalization ability of GCN. We search from a hyperparameter space and obtain the final one. The hyperparameter space is reported in Table~\ref{tab:parameter}. Ultimately, the parameters of each dataset for section 3 are determined as follows:
\begin{itemize}
\item GOODCora-degree-covariate: lr=1e-3, dropout=0.5,\newline hidden=200, model\_layer=2
\item GOODCora-degree-concept: lr=1e-3, dropout=0.5,\newline hidden=200, model\_layer=2
\item GOODCora-word-covariate: lr=1e-3, dropout=0.5,\newline hidden=300, model\_layer=2
\item GOODCora-word-concept: lr=1e-3, dropout=0.5,\newline hidden=300, model\_layer=1
\item GOODArxiv-degree-covariate: lr=1e-3, dropout=0.2,\newline hidden=300, model\_layer=3
\item GOODArxiv-degree-concept: lr=1e-3, dropout=0.2,\newline hidden=300, model\_layer=3
\item GOODArxiv-time-covariate: lr=1e-3, dropout=0.2,\newline hidden=300, model\_layer=3
\item GOODArxiv-time-concept: lr=1e-3, dropout=0.2,\newline hidden=300, model\_layer=3
\item GOODTwitch-language-covariate: lr=1e-3, dropout=0.5,\newline hidden=200, model\_layer=2
\item GOODTwitch-language-concept: lr=1e-3, dropout=0.5,\newline hidden=300, model\_layer=3
\item GOODWebKB-university-concept: lr=5e-3, dropout=0.5,\newline hidden=300, model\_layer=1
\end{itemize}

The other experiment is to compare our proposed model \method{} with other baselines.
We fine-tune the parameters within the following search space: layers (2, 3), dropout (0, 0.1, 0.2, 0.5), hidden (100, 200, 300),  $\gamma$ (0, 0.2, 0.5), $\beta$ (0, 0.1, 0.2, 0.5), learning rate (1e-3, 5e-2, 5e-3), heads (2, 4). Taking the GOODCora-degree-Concept dataset as an example, the parameters used for GCN are: layers: 2, learning rate: 5e-3, dropout: 0.2, hidden: 300; the parameters used for GAT are: layers: 2, learning rate: 5e-3, dropout: 0.2, hidden: 300, heads: 2; the parameters used for APPNP are: layers: 2, learning rate: 5e-3, dropout: 0.2, hidden: 300, $\beta$: 0.2; the parameters used for \method{} are: layers: 2, learning rate: 5e-3, dropout: 0.2, hidden: 300, heads: 2, $\gamma$: 0.5, $\beta$: 0.2. The magnitudes of the parameters remain consistent across various models

\subsection{Ablation Study and Hyperparameter Study} \label{sec:abl}
 We conducted the ablation study to analyze the impact of each component in \method. From Table~\ref{tab:abl}, we find that each component contributes positively to the model.

 We also conducted a hyperparameter study on $\gamma$ and $\beta$ to evaluate the impact of the choice of hyperparameters. The experiment is conducted on OGB-Arxiv dataset and results are summarized in Table~\ref{tab:paramstudy}, where we observed that the performance of \method demonstrates robustness to parameter variations.
 
\subsection{~\method{} Performance as a Backbone} \label{sec:backbone}
We conduct experiments to evaluate our model and baselines
across various strategies proposed for OOD. The results of
OOD performance under EERM setting on datasets from GOOD benchmark are reported in Table~\ref{table:ERRMsetting1}. The results of OOD performance under IRM, VREx, GroupDRO, and Graph-Mixup are respectively presented in Table~\ref{table:IRM}, Table~\ref{table:VREx},Table~\ref{table:GroupDRO} and Table~\ref{table:Graph-Mixup}.

\begin{figure*}[t]
    \begin{subfigure}{0.5\textwidth}
        \centering
        \includegraphics[width=\linewidth]{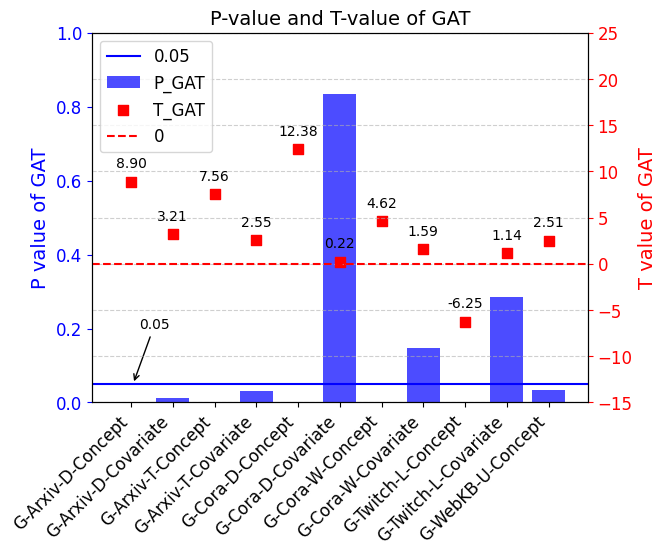}
        \caption{}
    \end{subfigure}\hfill
    \begin{subfigure}{0.5\textwidth}
        \centering
        \includegraphics[width=\linewidth]{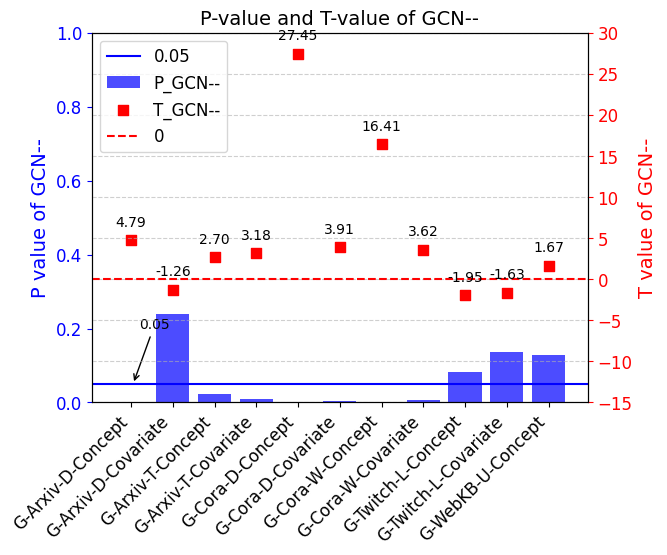}
        \caption{}
    \end{subfigure}\hfill
    \caption{$P_{value}$ and $T_{value}$ of GAT and GCN--}
    \label{fig:GAT_GCN}
\end{figure*}

\begin{figure*}[t]
    \begin{subfigure}{0.5\textwidth}
        \centering
        \includegraphics[width=\linewidth]{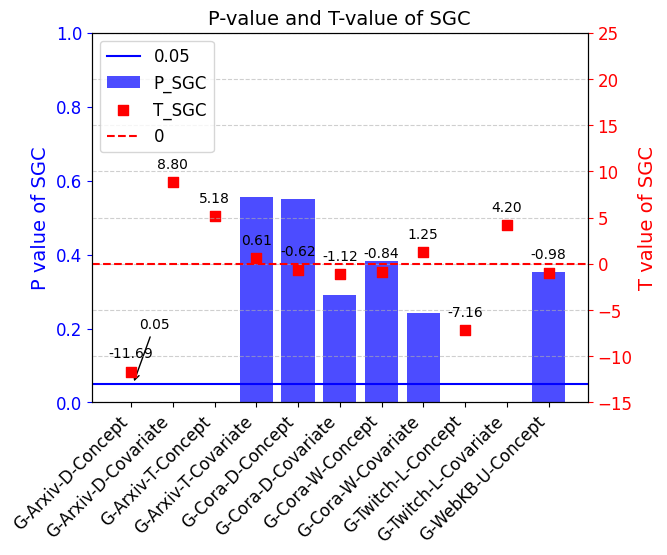}
        \caption{}
    \end{subfigure}\hfill
    \begin{subfigure}{0.5\textwidth}
        \centering
        \includegraphics[width=\linewidth]{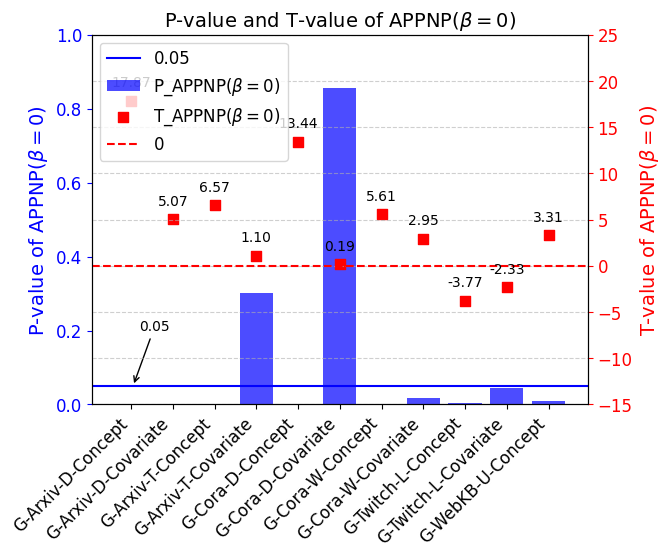}
        \caption{}
    \end{subfigure}\hfill
    \caption{$P_{value}$ and $T_{value}$ of SGC and APPNP($\beta=0$)}
    \label{fig: SGC}
\end{figure*}
\begin{table*}[]
\centering
\caption{Ablation study of \method on representative datasets on ERM setting. All numerical results are averages across 10 random runs.}
\label{tab:abl}
\begin{tabular}{@{}llll@{}}
\toprule                                  & GOOD-Cora-D-Covariate & elliptic            & OGB-Arxiv           \\ \midrule
\method{}                              & $\mathbf{59.68\pm0.46}$ & $\mathbf{73.09\pm2.14}$ & $\mathbf{45.51\pm0.67}$ \\
\method{}  w/o self-attention           & $59.19\pm0.58$          & $69.58\pm2.08$          & $45.29\pm0.83$          \\
\method{}  w/o decouple                 & $58.25\pm0.58$          & $70.26\pm1.46$          & $40.44\pm1.36$          \\
\method{}  w/o remove linear classifier & $57.52\pm1.04$          & $67.94\pm2.63$          & $44.79\pm0.63$          \\ \bottomrule     
\end{tabular}%
\end{table*}

\begin{table*}[]
\centering
\caption{Hyperparameter analysis of \method on OGB-Arxiv on ERM setting. All numerical results are averages across 10 random runs.}
\label{tab:paramstudy}
\begin{tabular}{@{}clllll@{}}
\toprule
\textbf{} &
  \multicolumn{1}{c}{0.1} &
  \multicolumn{1}{c}{0.2} &
  \multicolumn{1}{c}{0.3} &
  \multicolumn{1}{c}{0.4} &
  \multicolumn{1}{c}{0.5} \\ \midrule
\multicolumn{1}{l}{$\gamma$} &
  44.94 ± 0.71 &
  45.44 ± 0.64 &
  45.46 ± 0.54 &
  45.28 ± 0.92 &
  45.13 ± 0.51 \\
\textbf{$\beta$} &
  44.50 ± 0.51 &
  45.51 ± 0.67 &
  45.47 ± 0.91 &
  46.39 ± 0.16 &
  45.78 ± 0.47 \\ \bottomrule
\end{tabular}
\end{table*}

\begin{table*}[]
\caption{OOD performances on GOOD datasets under EERM setting.
All numerical results are averages across 10 random runs.}
\label{table:ERRMsetting1}
\resizebox{\textwidth}{!}{%
\begin{tabular}{@{}cllllllllllllll@{}}
\toprule
\multicolumn{1}{l}{} &  &  &\multicolumn{2}{c}{G-Cora-Word} & \multicolumn{2}{c}{G-Cora-Degree} & \multicolumn{2}{c}{G-Arxiv-Time} & \multicolumn{2}{c}{G-Arxiv-Degree} & \multicolumn{2}{c}{G-Twitch-Language} & \multicolumn{2}{c}{G-WebKB-University} \\ \cmidrule(l){2-15} 
\multicolumn{1}{l}{} & & & Covariate & Concept & Covariate & Concept& Covariate & Concept & Covariate & Concept & Covariate & Concept & Covariate & Concept \\ \midrule
\multirow{7}{*} & GCN-- &  & 65.38 & 65.81 & 58.02 & 16.22 & OOM & OOM & OOM & OOM & OOM & OOM &- & 25.14  \\ 
 & SGC & \textbf{} & OOM & 64.37 & OOM & 60.34 & OOM & OOM & OOM & OOM & OOM & OOM & - & 24.68 \\ 
 & APPNP &  & 64.53 & 65.81 & 57.59 & 64.18 & OOM & OOM & OOM & OOM & OOM & OOM & - & 33.48 \\ 
 & GAT & \textbf{} & 65.37 & 65.36 & 56.62 & 63.70 & OOM & OOM & OOM & OOM & OOM & OOM & \textbf{-} & 26.79 \\ 
 & GPRGNN &  & 64.57 & 65.09 & 58.16 & 64.71 & OOM & OOM & OOM & OOM & OOM & OOM & - & 34.22 \\  
 & \cellcolor{Gray} \method{} &\cellcolor{Gray}  &\cellcolor{Gray}\textbf{65.80} &\cellcolor{Gray}\textbf{65.99} &\cellcolor{Gray}\textbf{58.69} &\cellcolor{Gray}\textbf{65.00} &\cellcolor{Gray}OOM &\cellcolor{Gray}OOM &\cellcolor{Gray}OOM &\cellcolor{Gray}OOM &\cellcolor{Gray}OOM &\cellcolor{Gray}OOM &\cellcolor{Gray}- & \cellcolor{Gray}\textbf{36.06} \\ \midrule
\end{tabular}%
}
\end{table*}

\begin{table*}[]
\caption{OOD and GAP performances  on IRM setting. All numerical results are averages across 10 random runs.}
\label{table:IRM}
\scalebox{0.9}{%
\begin{tabular}{@{}clllllllllllll@{}}
\toprule
\multicolumn{1}{l}{} &  & \multicolumn{2}{c}{G-Cora-Word} & \multicolumn{2}{c}{G-Cora-Degree} & \multicolumn{2}{c}{G-Arxiv-Time} & \multicolumn{2}{c}{G-Arxiv-Degree} & \multicolumn{2}{c}{G-Twitch-Language} & \multicolumn{2}{c}{G-WebKB-University} \\ \midrule
\multicolumn{1}{l}{} & & OOD$\uparrow$ & GAP$\downarrow$ & OOD$\uparrow$ & GAP$\downarrow$ & OOD$\uparrow$ & GAP$\downarrow$ & OOD$\uparrow$ & GAP$\downarrow$ & 
OOD$\uparrow$ & GAP$\downarrow$ &
OOD$\uparrow$ & GAP$\downarrow$ \\ \midrule
\multirow{3}{*}{Covariate} & GCN– & 66.83 & \textbf{5.46} & 58.04 & 16.12 & \textbf{71.01} & 2.25 & 58.90 & 19.20 & 51.71 & 22.66 & - & - \\ 
 & APPNP & 67.18 & 6.52 & 58.31 & 17.15 & 69.03 & 2.56 & 56.20 & 20.87 & 56.64 & 16.81 & - & - \\ 
 &\cellcolor{Gray}\method &\cellcolor{Gray}\textbf{67.62} &\cellcolor{Gray}6.00 &\cellcolor{Gray}\textbf{60.40} &\cellcolor{Gray}\textbf{14.72} &\cellcolor{Gray}70.90 &\cellcolor{Gray}\textbf{1.98} &\cellcolor{Gray}\textbf{59.73} &\cellcolor{Gray}\textbf{17.72} &\cellcolor{Gray}\textbf{57.25} &\cellcolor{Gray}\textbf{15.43} &\cellcolor{Gray}-  &\cellcolor{Gray}- \\ \midrule
\multirow{3}{*}{Concept} & GCN– & 66.76 & \textbf{2.21} & 64.89 & 4.27 & 62.89 & 13.22 & 62.27 & 13.61 & 44.84 & 39.53 & 27.34 & \textbf{33.82} \\ 
 & APPNP & 67.29 & 4.10 & 65.39 & 5.10 & 63.55 & \textbf{10.76} & 63.97 & 8.88 & \textbf{48.13} & \textbf{34.05} & 26.79 & 43.21 \\  
 & \cellcolor{Gray}\method &\cellcolor{Gray}\textbf{67.43} &\cellcolor{Gray}3.45 &\cellcolor{Gray}\textbf{66.29} &\cellcolor{Gray}\textbf{4.22} &\cellcolor{Gray}\textbf{65.10} &\cellcolor{Gray}11.01 &\cellcolor{Gray}\textbf{65.62} &\cellcolor{Gray}\textbf{8.54} &\cellcolor{Gray}45.21 &\cellcolor{Gray}40.08 & \cellcolor{Gray}\textbf{32.93} &\cellcolor{Gray}38.56 \\ \bottomrule
\end{tabular}%
}
\end{table*}

\begin{table*}[]
\caption{OOD and GAP performances on VREx.
All numerical results are averages across 10 random runs.}
\label{table:VREx}
\resizebox{\textwidth}{!}{%
\begin{tabular}{@{}clllllllllllll@{}}
\toprule
\multicolumn{1}{l}{} &  & \multicolumn{2}{c}{G-Cora-Word} & \multicolumn{2}{c}{G-Cora-Degree} & \multicolumn{2}{c}{G-Arxiv-Time} & \multicolumn{2}{c}{G-Arxiv-Degree} & \multicolumn{2}{c}{G-Twitch-Language} & \multicolumn{2}{c}{G-WebKB-University} \\ \cmidrule(l){3-14} 
\multicolumn{1}{l}{} & & OOD$\uparrow$ & GAP$\downarrow$ & OOD$\uparrow$ & GAP$\downarrow$ & OOD$\uparrow$ & GAP$\downarrow$ & OOD$\uparrow$ & GAP$\downarrow$ & OOD$\uparrow$ & GAP$\downarrow$ & OOD$\uparrow$ & GAP$\downarrow$ \\ \midrule
\multirow{3}{*}{Covariate} & GCN– & 66.52 & \textbf{5.6} & 59.08 & 15 & 69.93 & \textbf{1.27} & 58.78 & 19.18 & 51.09 & 21.73 & -& - \\ 
 & APPNP & 67.31 & 6.78 & 60.38 & 15.57 & 70.22 & 2.47 & 56.35 & 20.71 & \textbf{57.85} & \textbf{12.9} & - & - \\ 
 &\cellcolor{Gray}\method &\cellcolor{Gray}\textbf{67.68} & \cellcolor{Gray}6.30 & \cellcolor{Gray}\textbf{60.50} &\cellcolor{Gray}\textbf{14.34} &\cellcolor{Gray}\textbf{71.15} &\cellcolor{Gray}1.53 &\cellcolor{Gray}\textbf{60.44} &\cellcolor{Gray}\textbf{17.38} &\cellcolor{Gray}57.38 &\cellcolor{Gray}14.23 &\cellcolor{Gray}- &\cellcolor{Gray}- \\ \midrule
\multirow{3}{*}{Concept} & GCN– & 66.54 & \textbf{2.45} & 64.78 & 4.36 & 62.78 & 13.26 & 61.86 & 12.5 & 48.89 & 33.79 & 27.70 & \textbf{34.46} \\ 
 & APPNP & 67.36 & 3.97 & 66.14 & 4.96 & 63.50 & 10.84 & 63.23 & 7.96 & \textbf{52.14} & \textbf{23.43} & 32.84 & 38.82 \\ 
 & \cellcolor{Gray}\method &\cellcolor{Gray}\textbf{67.53} &\cellcolor{Gray}3.44 & \cellcolor{Gray}\textbf{66.49} &\cellcolor{Gray}\textbf{4.26} &\cellcolor{Gray}\textbf{64.94} &\cellcolor{Gray}\textbf{10.42} &\cellcolor{Gray}\textbf{65.07} &\cellcolor{Gray}\textbf{7.41} &\cellcolor{Gray}48.27 &\cellcolor{Gray}33.81 &\cellcolor{Gray}\textbf{33.76} &\cellcolor{Gray}38.07 \\ \bottomrule
\end{tabular}%
}
\end{table*}

\begin{table*}[]
\caption{OOD and GAP performances on GroupDRO.
All numerical results are averages across 10 random runs.}
\label{table:GroupDRO}
\resizebox{\textwidth}{!}{%
\begin{tabular}{@{}llllllllllllll@{}}
\toprule
\multicolumn{1}{l}{} &  & \multicolumn{2}{c}{G-Cora-Word} & \multicolumn{2}{c}{G-Cora-Degree} & \multicolumn{2}{c}{G-Arxiv-Time} & \multicolumn{2}{c}{G-Arxiv-Degree} & \multicolumn{2}{c}{G-Twitch-Language} & \multicolumn{2}{c}{G-WebKB-University} \\ \multicolumn{1}{l}{} & & OOD$\uparrow$ & GAP$\downarrow$ & OOD$\uparrow$ & GAP$\downarrow$ & OOD$\uparrow$ & GAP$\downarrow$ & OOD$\uparrow$ & GAP$\downarrow$ & OOD$\uparrow$ & GAP$\downarrow$ & OOD$\uparrow$ & GAP$\downarrow$ \\ \midrule
\multirow{3}{*}{Covariate} & GCN– & 66.73 & \textbf{5.48} & 57.97 & 16.05 & 70.31 & 2.98 & 58.7 & 19.22 & 52.07 & 22.34 & - & - \\ 
 & APPNP & 67.57 & 6.39 & 58.78 & 16.94 & 68.78 & 2.87 & 56.08 & 20.87 & \textbf{57.64} & \textbf{13.66} & - & - \\ 
 & \cellcolor{Gray}\method &\cellcolor{Gray}\textbf{67.65} &\cellcolor{Gray}6.58 &\cellcolor{Gray}\textbf{59.69} &\cellcolor{Gray}\textbf{15.83} &\cellcolor{Gray}\textbf{70.88} &\cellcolor{Gray}\textbf{1.80} &\cellcolor{Gray}\textbf{60.08} &\cellcolor{Gray}\textbf{1.76} &\cellcolor{Gray}57.24 &\cellcolor{Gray}15.35 &\cellcolor{Gray}- &\cellcolor{Gray}- \\ \midrule
\multirow{3}{*}{Concept} & GCN– & 66.30 & \textbf{2.69} & 64.83 & 4.14 & 62.92 & 13.08 & 62.43 & 13.38 & 44.76 & 39.57 & 28.07 & \textbf{34.09} \\ 
 & APPNP & 67.33 & 4.07 & 66.37 & 4.94 & 63.50 & 10.81 & 64.17 & \textbf{8.61} & \textbf{51.31} & \textbf{24.37} & 32.48 & 39.02 \\ 
 & \cellcolor{Gray}\method &\cellcolor{Gray}\cellcolor{Gray}\textbf{67.54} & \cellcolor{Gray}6.45 &\cellcolor{Gray}\textbf{66.49} &\cellcolor{Gray}\textbf{4.39} &\cellcolor{Gray}\textbf{64.42} &\cellcolor{Gray}\textbf{10.60} &\cellcolor{Gray}\textbf{64.92} &\cellcolor{Gray}8.81 &\cellcolor{Gray}44.75 &\cellcolor{Gray}40.65 &\cellcolor{Gray}\textbf{33.19} &\cellcolor{Gray}38.88 \\ \bottomrule
\end{tabular}%
}
\end{table*}


\begin{table*}[]
\caption{OOD and GAP performances on Graph-Mixup.
All numerical results are averages across 10 random runs. }
\label{table:Graph-Mixup}
\resizebox{\textwidth}{!}{%
\begin{tabular}{@{}clllllllllllll@{}}
\toprule
\multicolumn{1}{l}{} &  & \multicolumn{2}{c}{G-Cora-Word} & \multicolumn{2}{c}{G-Cora-Degree} & \multicolumn{2}{c}{G-Arxiv-Time} & \multicolumn{2}{c}{G-Arxiv-Degree} & \multicolumn{2}{c}{G-Twitch-Language} & \multicolumn{2}{c}{G-WebKB-University} \\ \cmidrule(l){3-14} 
\multicolumn{1}{l}{} & & OOD$\uparrow$ & GAP$\downarrow$ & OOD$\uparrow$ & GAP$\downarrow$ & OOD$\uparrow$ & GAP$\downarrow$ & OOD$\uparrow$ & GAP$\downarrow$ & OOD$\uparrow$ & GAP$\downarrow$ & OOD$\uparrow$ & GAP$\downarrow$ \\ \midrule
\multirow{3}{*}{Covariate} & GCN & 64.16 & 8.20 & 55.28 & 18.44 & 67.00 & 2.34 & 53.94 & 20.75 & 57.55 & 14.73 & - & - \\
 & APPNP & 65.12 & 8.43 & 56.91 & 17.55 & 63.04 & 10.95 & 55.36 & 21.33 & \textbf{58.12} & \textbf{13.43} & - & - \\
 &\cellcolor{Gray}\method &\cellcolor{Gray}\textbf{66.00} &\cellcolor{Gray}\textbf{7.65} &\cellcolor{Gray}\textbf{58.10} &\cellcolor{Gray}\textbf{16.59} &\cellcolor{Gray}\textbf{70.11} &\cellcolor{Gray}\textbf{2.16} &\cellcolor{Gray}\textbf{58.03} &\cellcolor{Gray}\textbf{19.50} &\cellcolor{Gray}57.61 &\cellcolor{Gray}15.26 &\cellcolor{Gray}- &\cellcolor{Gray}- \\ \midrule
\multirow{3}{*}{Concept} & GCN & 64.69 & 5.49 & 63.41 & 7.37 & 60.31 & 11.59 & 55.36 & 13.67 & 47.64 & 35.55 & 30.64 & 43.19 \\
 & APPNP & 66.91 & \textbf{4.15} & 65.10 & \textbf{4.87} & 63.05 & 10.95 & 63.61 & \textbf{8.91} & \textbf{51.17} & \textbf{28.07} & 26.51 & 39.82 \\
 & \cellcolor{Gray}\method &\cellcolor{Gray}\textbf{67.11} &\cellcolor{Gray}4.04 &\cellcolor{Gray}\textbf{65.22} &\cellcolor{Gray}5.35 &\cellcolor{Gray}\textbf{64.06} &\cellcolor{Gray}\textbf{10.85} &\cellcolor{Gray}\textbf{64.38} &\cellcolor{Gray}9.32 &\cellcolor{Gray}45.55 &\cellcolor{Gray}38.77 & \cellcolor{Gray}\textbf{31.10} &\cellcolor{Gray}\textbf{39.39} \\ \bottomrule
\end{tabular}%
}
\end{table*}

\end{document}